\pdfoutput=1
\documentclass[letterpaper, 10 pt, conference]{ieeeconf}  

\IEEEoverridecommandlockouts                              

\usepackage{cite}
\usepackage{amsmath,amssymb,amsfonts}

\usepackage{amsthm}
\usepackage{algorithm}
\usepackage{algorithmic}
\usepackage{graphicx}
\usepackage{textcomp}
\usepackage{xcolor}
\newtheorem{theorem}{Theorem}
\newtheorem{lemma}{Lemma}

\usepackage{subcaption}
\usepackage{xurl}

\theoremstyle{remark}

\theoremstyle{definition}

\def\BibTeX{{\rm B\kern-.05em{\sc i\kern-.025em b}\kern-.08em
    T\kern-.1667em\lower.7ex\hbox{E}\kern-.125emX}}

\begin{document}

\title{Guiding Vector Field Generation via Score-based Diffusion Model}

\author{Zirui Chen\textsuperscript{1,2}, Shiliang Guo\textsuperscript{2} and Shiyu Zhao\textsuperscript{2}
    \thanks{Acknowledgments: This work was supported by the Brain Science and Brain-like Intelligence Technology — National Science and Technology Major Project (Grant No. 2022ZD0208800), Corresponding author: Shiyu Zhao.}
    \thanks{\textsuperscript{1}College of Computer Science and Technology, Zhejiang University, Hangzhou, China. \textsuperscript{2}WINDY Lab, Department of Artificial Intelligence, Westlake University, Hangzhou, China. \{chenzirui, guoshiliang, zhaoshiyu\}@westlake.edu.cn.}
    }

\maketitle

\begin{abstract}
Guiding Vector Fields (GVFs) are a powerful tool for robotic path following. However, classical methods assume smooth, ordered curves and fail when paths are unordered, multi-branch, or generated by probabilistic models. We propose a unified framework, termed the Score-Induced Guiding Vector Field (SGVF), which leverages score-based generative modeling to construct vector fields directly from data distributions. SGVF learns tangent fields from point clouds with unit-norm, orthogonality, and directional-consistency losses, ensuring geometric fidelity and control feasibility. This approach removes the reliance on ad-hoc path segmentation and enables guidance along complex topologies such as branching and pseudo-manifolds. The study establishes a correspondence between score vanishing in diffusion models and GVF singularities and highlights representational capacity near sharp path curvatures. Experiments on robotic navigation in planar environments demonstrate that SGVF achieves reliable path following in scenarios where classical GVFs fail, underscoring its potential as a bridge between generative modeling and geometric control. Code and experiment video are available at \url{https://github.com/czr-gif/Guiding-Vector-Field-Generation-via-Score-based-Diffusion-Model}.
\end{abstract}

\section{Introduction}
Guiding Vector Field (GVF) methods have emerged as a powerful framework in robotic guidance and control. By establishing a connection between task-level goals and motion-level control, GVFs allow for a geometric interpretation of control behavior. In particular, GVFs have been extensively applied to path following problems \cite{chen2025non,Hu_2023_05,Rezende_2022_04,zhou2025inverse,singh2024state}, where a desired path is converted into a vector field that maps the geometry of the path to the desired velocity of the agent.

Traditionally, the paths to be followed are specified either by analytical expressions or as discrete data sets. For analytically-defined paths, the topology of the curve plays a crucial role. For instance, the homology of a path—whether it is line-like or circle-like—determines whether the resulting GVF will contain singularities or not, as studied in \cite{yao2022guiding}. On the other hand, when the path is given as a data sequence, interpolation techniques are employed to generate a smooth, parameterized curve. Common choices include Lagrange interpolation \cite{Goncalves_2010} and Fourier-based interpolation methods \cite{Chen_2025_01}.

Interestingly, certain topological transformations can map a closed path (e.g., a circle) to an open, line-like structure in a higher-dimensional space. This allows the path to be treated as a parameterized trajectory \cite{Yao_2021Singularity}, enabling control schemes that interpret path following as a continuous pursuit process—where the robot ``chases” a moving reference point along the curve. This perspective builds a conceptual bridge between classical path following and trajectory tracking.

However, such parameterizations rely on the sequential structure of path points. When the way points are either unordered or contain multiple branches (see Fig.\ref{fig:index}) —as is common in generative scenarios like diffusion-based point cloud generation \cite{luo2021diffusion}—this structure breaks down. The reference point effectively ``jumps" between discontinuous locations, as illustrated in Fig. \ref{fig:index}, destroying the geometric continuity required by traditional GVF . Moreover, when the path itself forms a piecewise-manifold (such as a square or a triangle), constructing a unified guiding vector field becomes impossible, since the vector field heavily depends on the path’s differential.

\begin{figure}
    \centering
    \includegraphics[width=1\linewidth]{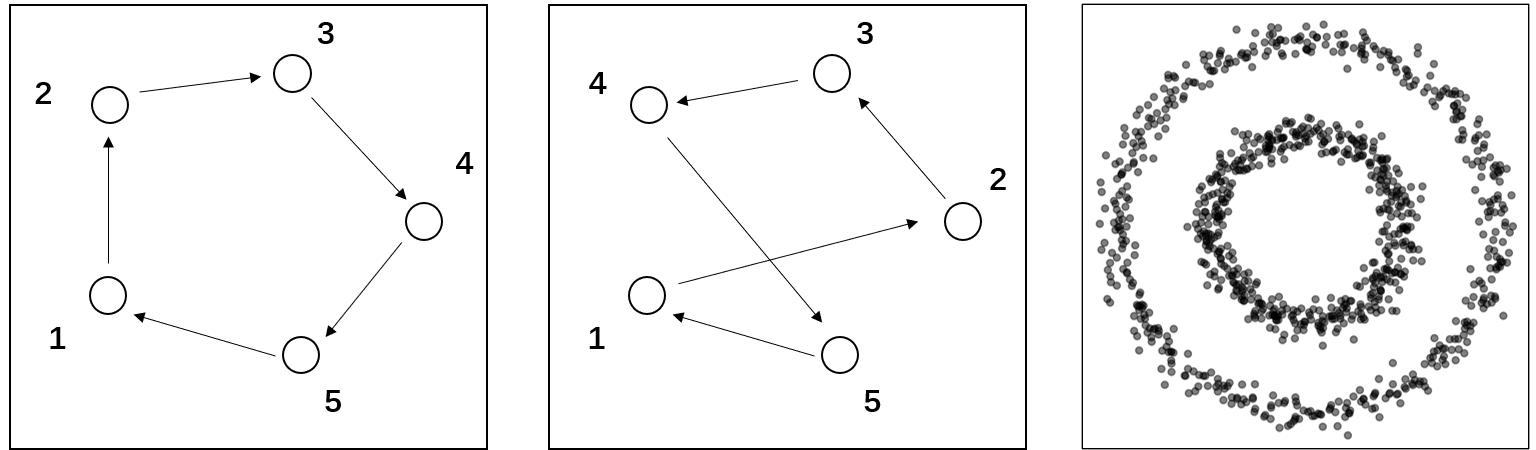}
    \caption{Path indexing illustration: left – orderly indexing preserves the waypoint shape; middle – shuffled points break the interpolation result; right – random points with branching waypoints cannot be sequentially indexed.}
    \label{fig:index}
\end{figure}

Nevertheless, such paths hold significant potential for future robotic path planning. The emergence of powerful generative models —such as diffusion-based models— has opened new directions in robot planning and control. For instance, diffusion policies and generative trajectory planners (e.g., \cite{xiao2023safediffuser, Han_2025_06, 11027664}) have shown remarkable success in high-level task planning. However, these models typically operate on distributions rather than parameterized curves, creating a mismatch between classical GVF design and modern data-driven robotics.

Classical GVF methods typically assume that the target path is given as a smooth, ordered, and parameterized curve \cite{Sujit_2014Unmanned}. However, such assumptions often break down in modern robotics applications, where paths are increasingly generated by probabilistic or generative models in the form of unordered or sparsely sampled point clouds \cite{luo2021diffusion, xiao2023safediffuser, Han_2025_06}. These data-driven representations lack explicit sequential structure, making it difficult for conventional GVF designs to preserve geometric continuity or generate feasible control actions.

It is interesting to note that a natural connection emerges between diffusion-based generative models and guiding vector fields: diffusion models describe how data points drift from a noise distribution toward a target distribution, with each step providing a drifting term that directs the data toward higher-probability regions. In score-based generative models, this drifting direction is formally given by the score function, i.e., the gradient of the log-probability density \cite{song2020score}. When path points are interpreted as samples from a target distribution, the score function naturally defines a vector field over space that follows the underlying data manifold. In this sense, the score function directly corresponds to the gradient component of a guiding vector field, offering a principled bridge between probabilistic generative modeling and geometric path guidance.

Motivated by this observation, we reinterpret path following as a distributional guidance problem, where vector fields are constructed directly on the support of the data distribution. Leveraging this perspective, we introduce the \textbf{Score-Induced Guiding Vector Field (SGVF)}, a novel framework that integrates generative modeling with vector field design. SGVF retains the structure-awareness of classical control while extending its applicability to unstructured, multi-modal, and probabilistic path representations in a unified and robust manner.

This paper makes three key contributions: 1) A unified framework, termed SGVF, that enables robots to follow complex paths---including unordered, branching, and probabilistic paths---by circumventing the reliance on ad-hoc segmentation; 2) an approach that infers intrinsic manifold geometry directly from raw point clouds, utilizing learned score fields to approximate tangent directions and encode control-relevant information; and 3) a theoretical correspondence between score vanishing in diffusion models and singularities in GVFs, offering new insights into topological robustness and strategies to mitigate degeneracies.

\section{Problem Formulation}

\subsection{Path Following on Discrete Waypoints}
In this work, we consider the \emph{path following problem}, where a robot must converge to and move along a desired path in space without adhering to a prescribed timing law. Specifically, this path is represented by a set of discrete points, termed point clouds, whose samples are denoted as $p_{\text{data}}$.

This general formulation encompasses classical paths as a special case. In the classical setting, a path is often defined implicitly as the zero level set of a smooth function $\phi:\mathbb{R}^d \to \mathbb{R}$, i.e.,
\[
\mathcal{P} = \{ x \in \mathbb{R}^d \mid \phi(x) = 0 \},
\]
and a guiding vector field (GVF) $\mathbf{u}(x) \in \mathbb{R}^d$ is constructed from explicit tangent and normal directions, e.g.,
\[
\mathbf{u}(x) = \mathbf{T}(x) - k_n \nabla \phi(x)\, \phi(x),
\]
with $\mathbf{T}(x)$ a tangent vector, $\nabla \phi(x)$ a gradient vector, and $k_n>0$ a gain. Such designs are effective for smooth curves but become unreliable when the path consists of unordered, discrete, or multi-branched points, where differential structures are undefined.

\subsection{Score-based Diffusion Formulation}
To address the limitations of classical explicit definitions, we adopt a data-driven formulation based on \emph{Score-based Diffusion Models}. We treat the path waypoints as samples drawn from an underlying data distribution $p_{\text{data}}(x)$. The diffusion model defines a forward stochastic process that gradually perturbs the data into noise via a Stochastic Differential Equation (SDE):
$$
dx = \sqrt{\frac{d[\sigma^2(t)]}{dt}} d\mathbf{w}, \quad t \in [0,1],
$$
where $\mathbf{w}$ is the standard Wiener process and $\sigma(t)$ is a noise schedule. This process yields a time-dependent marginal distribution $p_t(x)$. The geometry of the data manifold is encoded in the \emph{score function} $\mathbf{S}(x, t) \in \mathbb{R}^d$, defined as the gradient of the log-density:
$$
\mathbf{S}(x, t) = \nabla_{x} \log p_t(x).
$$
Intuitively, $\mathbf{S}(x, t)$ defines a vector field pointing towards high-density regions (i.e., the path). 

From a statistical perspective, the score function satisfies \emph{Stein’s Identity} \cite{Stein1972,Chen2011}. For any smooth function vector field $\mathbf{f}(x)$ with compact support, the following holds under distribution $p(x)$:
$$
\mathbb{E}_{p(x)}[\mathbf{S}(x)\, \mathbf{f}(x)^T] = - \mathbb{E}_{p(x)}[\nabla_{x} \mathbf{f}(x)].
$$
By substituting $\mathbf{f}(x) = \log p(x)$, this yields a natural orthogonality property in expectation:
$$
\mathbb{E}_{p(x)}[\mathbf{S}(x)\, \log p(x)] = 0.
$$
This identity underpins modern estimation methods like score matching \cite{Hyvarinen2005} and reveals the score's connection to Fisher information geometry \cite{Amari2000}. This expected orthogonality highlights the score as a distribution-driven descriptor that inherently captures geometric structure, serving as the theoretical foundation for our network training without relying on explicit parameterization.

\section{methodology}
This section describes the construction and training process of the SGVF. The SGVF for path following is constructed from two complementary components: a gradient component that attracts the robot toward the path and a tangent component that induces motion along the path. Formally, for a given waypoints set $\{x_i\}_{i=1}^N$, the SGVF $\mathbf{m}(x) \in \mathbb{R}^d$ at a point $x$ is expressed as
\[
\mathbf{m}(x) = \mathbf{u}_\text{grad}(x) + \mathbf{u}_\text{tan}(x),
\]
where $\mathbf{u}_\text{grad}(x)\in \mathbb{R}^d$ represents the gradient component and $\mathbf{u}_\text{tan}(x)\in \mathbb{R}^d$ represents the tangent component. 

In our approach, the gradient component is computed using the \emph{probability gradient} (i.e., the score) learned from the observed waypoints:
\[
\mathbf{u}_\text{grad}(x) = \mathbf{s}_\theta(x, t=1) \approx \nabla_x \log p(x \mid \{x_i\}),
\]
where $\mathbf{s}_\theta(x,t) \in \mathbb{R}^d$ denotes a transformed score vector that preserves the original score direction while being suitable for constructing the guiding vector field, and $p(x \mid \{x_i\})$ is the implicit distribution defined by the waypoints $\{x_i\}$. This component captures the local attraction toward regions of high probability along the path. The tangent component is obtained heuristically on top of the learned score, either through training with trajectory-following objectives or other guidance heuristics, to induce progression along the path:
\[
\mathbf{u}_\text{tan}(x) = \mathbf{v}_\phi(x),
\]
where $\mathbf{v}_\phi(x)\in \mathbb{R}^d$ is learned from a neural network. By combining these two components, the resulting vector field $\mathbf{m}(x)$ can drive the robot to follow the path defined by the discrete waypoints, handling both attraction to the path and forward motion along it.

\subsection{Score Field Learning}
The objective is to learn the gradient field of the data distribution, $\nabla_x \log p_{\text{data}}(x)$, to guide the robot. Direct estimation of this quantity is intractable. However, by adopting the Denoising Score Matching framework \cite{song2020score}, we can bypass density estimation and train the network using a tractable \emph{conditional score field} \cite{flowsanddiffusions2025}.

The training procedure is detailed in \textbf{Algorithm \ref{alg:score_conditional_noisy}}. Instead of relying on ordered sequences, we treat the path as a collection of independent waypoints. In each iteration, we sample a clean waypoint $z$ and perturb it with Gaussian noise $\sigma(t)\epsilon$ to obtain a noisy sample $x$. Crucially, for Gaussian perturbation kernels, the \emph{true conditional score} $u_r = \nabla_x \log p(x | z)$ has a closed-form solution:
\begin{equation}
    u_r = \frac{-(x - z)}{\sigma^2(t)} = -\frac{\epsilon}{\sigma(t)}.
\end{equation}
The network $\mathbf{S}_\theta(x,t)$ is trained to regress this target $u_r$. This formulation ensures that the learned field points effectively towards the data manifold (the path) from any noisy state, enabling robust guidance even for unordered or branching structures without explicit topological priors.

\begin{algorithm}[t]
\small 
\caption{Training Score Field via Denoising Score Matching}
\label{alg:score_conditional_noisy}
\begin{algorithmic}[1]
\REQUIRE Waypoints $\{x_i\}_{i=1}^N$, Noise schedule $\sigma(t)$, Network $\mathbf{S}_\theta$
\ENSURE Trained score field $\mathbf{S}_\theta(x,t)$

\FOR{each training iteration}
    \STATE Sample batch of waypoints $z \sim \{x_i\}_{i=1}^N$
    \STATE Sample time $t \sim \mathcal{U}(0,1)$ and noise $\epsilon \sim \mathcal{N}(0, \mathbf{I})$
    \STATE \textbf{Perturb:} $x = z + \sigma(t)\epsilon$ \quad \textit{// Forward Diffusion}
    \STATE \textbf{Target:} $u_r = -\epsilon / \sigma(t)$ \quad \textit{// Analytical Score Formula}
    \STATE Compute Loss: $\ell = \frac{1}{B} \sum \| \mathbf{S}_\theta(x,t) - u_r \|^2$
    \STATE Update $\theta$ via backpropagation to minimize $\ell$
\ENDFOR

\RETURN Trained score field $\mathbf{S}_\theta(x,t)$
\end{algorithmic}
\end{algorithm}

\subsection{Learning Orthogonal Tangent Fields}

We now consider training an auxiliary tangent vector field that, together with the score field, forms a composite vector field over the data manifold. The goal is to construct a tangent field that is geometrically meaningful, directionally smooth, and orthogonal to the score field almost everywhere. We achieve this by minimizing a combination of three losses, described below.

We denote by $\mathbf{S}_\theta(x,t)$ the raw score predicted by the network, 
and by $\mathbf{s}_\theta(x,t)$ its normalized version used for constructing the tangent component of the guiding vector field. 
The normalization is defined as
\begin{equation}
    \mathbf{s}_\theta(x,t) 
    = \tanh\!\big(k_s\|\mathbf{S}_\theta(x,t)\|\big)\,
      \frac{\mathbf{S}_\theta(x,t)}{\|\mathbf{S}_\theta(x,t)\|}.
\end{equation}
This transformation preserves the directional information of the original score while bounding its magnitude, which is crucial for stable vector field generation. 
Although the transformed score no longer strictly satisfies Stein's identity or zero-mean property, its direction remains consistent with the original $\mathbf{S}_\theta(x,t)$, allowing us to leverage the statistical information encoded in the raw score to guide the robot along the path.

Let $\mathbf{s}_\theta(x, t)$ be the learned score field and $\mathbf{v}_\phi(x, t)$ be the learnable tangent field. We define a mixed field as their linear combination:
\[
\mathbf{m}(x) = \mathbf{s}_\theta(x, t=1) + \mathbf{v}_\phi(x).
\]
To train $\mathbf{v}_\phi$, we define the following composite loss:
\[
\mathcal{L}_{\text{total}} = \lambda_1 \mathcal{L}_{\text{unit}} + \lambda_2 \mathcal{L}_{\text{orth}} + \lambda_3 \mathcal{L}_{\text{dir}},
\]
where each term encourages a different geometric property of the tangent field.

\paragraph{Unit-Length Constraint ($\mathcal{L}_{\text{unit}}$)}

We enforce that the composite vector $\mathbf{m}(x, t)$ lies on the unit sphere:
\[
\mathcal{L}_{\text{unit}} = \mathbb{E}_{x} \left[ \left( \left\| \mathbf{s}_\theta(x, 1) + \mathbf{v}_\phi(x) \right\|_2 - 1 \right)^2 \right].
\]
This loss serves two purposes: (i) it normalizes the overall guidance field and (ii) it implicitly balances the scale between score and tangent components.

\paragraph{Orthogonality Constraint ($\mathcal{L}_{\text{orth}}$)}

To promote the tangent field as locally orthogonal to the score, we minimize the cosine similarity between the two:
\[
\mathcal{L}_{\text{orth}} = \mathbb{E}_{x} \left[ \left\langle \frac{\mathbf{s}_\theta(x, 1)}{\|\mathbf{s}_\theta(x, 1)\|}, \frac{\mathbf{v}_\phi(x)}{\|\mathbf{v}_\phi(x)\|} \right\rangle^2 \right].
\]
This softly enforces orthogonality while preserving differentiability.

\paragraph{Directional Consistency Loss ($\mathcal{L}_{\text{dir}}$)}

Finally, to ensure that the mixed vector field does not oscillate arbitrarily, we introduce a smoothness term based on local directional consistency. For each point $x$, we compute cosine similarity with its $k$ neighbors (computed by adding noise on the point $x$) in the tangent field:
\[
\mathcal{L}_{\text{dir}} = \mathbb{E}_{x} \left[ 1 - \frac{1}{k} \sum_{x' \in \mathcal{N}_k(x)} \cos \angle(\mathbf{v}_\phi(x), \mathbf{v}_\phi(x')) \right],
\]
where $\mathcal{N}_k(x)$ denotes the $k$ neighbors of $x$ in Euclidean space. This term encourages local smoothness and reduces directional noise.

During training, we sample points $x = \gamma(z, t)$ from the conditional path and compute all three losses jointly. Gradients are only backpropagated through the tangent field $\mathbf{v}_\phi$, while the score field is kept fixed. The balancing weights $\lambda_1, \lambda_2, \lambda_3$ are selected empirically.

This design ensures that the learned tangent field is geometrically meaningful, smooth, and structurally aligned with the score field, paving the way for advanced applications such as surface-following generative paths or guided manifold traversal.

\begin{algorithm}[t]
\small
\caption{Training Tangent Field}
\label{alg:tangent}
\begin{algorithmic}[1]
\REQUIRE Trained score field $\mathbf{s}_\theta(x,t)$, Tangent network $\mathbf{v}_\phi$
\ENSURE Trained tangent field $\mathbf{v}_\phi(x)$

\FOR{each training iteration}
    \STATE Sample batch $x$ and nearest neighbors $\mathcal{N}_k(x)$
    \STATE \textbf{Inference:} $\mathbf{s} = \mathbf{s}_\theta(x,1)$, $\mathbf{v} = \mathbf{v}_\phi(x)$
    \STATE Form mixed field: $\mathbf{m} = \mathbf{s} + \mathbf{v}$
    \STATE \textbf{Compute Losses:}
    \STATE \quad Unit-length: $\ell_{unit} = (\|\mathbf{m}\|_2 - 1)^2$
    \STATE \quad Orthogonality: $\ell_{orth} = \langle \hat{\mathbf{s}}, \hat{\mathbf{v}} \rangle^2$
    \STATE \quad Consistency: $\ell_{dir} = 1 - \frac{1}{k}\sum_{x'\in \mathcal{N}_k(x)} \cos\angle(\mathbf{v}(x), \mathbf{v}(x'))$
    \STATE Update $\phi$ to minimize $\mathcal{L}_{total} = \lambda_1 \ell_{unit} + \lambda_2 \ell_{orth} + \lambda_3 \ell_{dir}$
\ENDFOR

\RETURN Trained tangent field $\mathbf{v}_\phi(x)$
\end{algorithmic}
\end{algorithm}

\section{Lyapunov Analysis of the Mixed Field}
This section presents a Lyapunov-based analysis of the stability properties of the proposed SGVF. For $x \in \mathbb{R}^n$, we analyze the continuous-time dynamics
\begin{equation}
\dot x = \mathbf{m}(x, t) \;\triangleq\; \mathbf{s}_\theta(x, 1)+\mathbf{v}_\phi(x),
\label{eq:dynamics}
\end{equation}
where $\mathbf{s}_\theta(x,1)=\nabla_x \log P(x)$ is the learned terminal-time score field and $v(x)$ is a learned tangent field trained with unit-norm, orthogonality, and directional-consistency constraints. Let $P:\mathbb{R}^d\to(0,\infty)$ be a smooth density with bounded supremum $P^\star \triangleq \sup_x P(x) < \infty$.
Define the Lyapunov function
\begin{equation}
V(x) \;\triangleq\; -\log\!\big(\tfrac{P(x)}{P^\star}\big)
\;=\; -\log P(x) + \log P^\star \;\ge 0,
\label{eq:lyapunov}
\end{equation}
whose zero-level set coincides with the set of global modes $\mathcal{M} \triangleq \{x:\,P(x)=P^\star\}$.
Note that $\nabla V(x) = -\nabla \log P(x) = -\mathbf{s}_\theta(x,1)$.

To establish that a sufficiently trained SGVF can accomplish the path-following task, we adopt the following assumptions, whose plausibility is justified by the training losses:
\begin{itemize}
\item[(A1)] $P$ is of class $C^1$, and its score $\mathbf{s}_\theta(\cdot,1)$ is locally Lipschitz.
\item[(A2)] After sufficient training, the learned tangent field $\mathbf{v}_\phi(x)$ and the score field $\mathbf{s}_\theta(x,1)$ approximately satisfy the following geometric properties: \\
\hspace{1em}(i) \emph{Orthogonality:} $\mathbf{s}_\theta(x,1)^\top \mathbf{v}_\phi(x) = 0$ for all $x$;\\
\hspace{1em}(ii) \emph{Unit length:} $\| \mathbf{m}(x) \| = \| \mathbf{s}_\theta(x,1) + \mathbf{v}_\phi(x) \| = 1$ for all $x$;\\
\hspace{1em}(iii) \emph{Directional consistency:} there exists a constant $L>0$ such that for any two points $x, y$ lying on the same level set of $P$, the tangent field satisfies
\[
\|\mathbf{v}_\phi(x) - \mathbf{v}_\phi(y)\| \le L \|x - y\|,
\]
ensuring local Lipschitz continuity of $\mathbf{v}_\phi$ along the level sets of $P$ and preventing abrupt changes in the guidance direction.
\item[(A3)] Trajectories remain within a compact sublevel set $\{ x : V(x) \le c \}$.
\end{itemize}

\begin{lemma}[Lyapunov decrease]
\label{lem:decrease}
Along any solution of \eqref{eq:dynamics}, the derivative of $V$ satisfies
\begin{equation}
\begin{aligned}
    \dot V(x) &\;\triangleq\; \nabla V(x)^\top \dot{x} \\
              &= \big(-\mathbf{s}_\theta(x,1)\big)^\top \big(\mathbf{s}_\theta(x,1) + \mathbf{v}_\phi(x)\big) \\
              &= -\|\mathbf{s}_\theta(x,1)\|^2 - \mathbf{s}_\theta(x,1)^\top \mathbf{v}_\phi(x).
\end{aligned}
\end{equation}

Under (A2.i), $\mathbf{s}_\theta(x,1)^\top \mathbf{v}_\phi(x)=0$, hence
\begin{equation}
\dot V(x) \;=\; -\|\mathbf{s}_\theta(x,1)\|^2 \;\le\; 0,
\end{equation}
with equality if and only if $\mathbf{s}_\theta(x,1)=0$.
\end{lemma}

\begin{proof}
The identity for $\dot V$ follows from $\nabla V=-\mathbf{s}_\theta(x,1)$ and definition of $m$.
Assumption (A2.i) yields orthogonality, hence the cross term vanishes and $\dot V=-\|\mathbf{s}_\theta\|^2\le 0$, with equality iff $\mathbf{s}=0$.
\end{proof}

\begin{theorem}
\label{thm:main}
Let assumptions (A1)--(A3) hold. Then:
\begin{itemize}
\item[(i)] $V$ is a Lyapunov function for the dynamics \eqref{eq:dynamics}.
All trajectories converge to the largest invariant set contained in
\[
\mathcal{E} \;\triangleq\; \{x:\ \dot V(x)=0\}
\;=\; \{x:\ \mathbf{s}_\theta(x,1)=0\}
\;\supseteq\; \mathcal{M}.
\]
\item[(ii)] On any level set of $V$ (in particular on $V^{-1}(0)=\mathcal{M}$), the vector field $\mathbf{v}_\phi$ is tangent:
$\mathbf{s}_\theta(x,1)^\top \mathbf{v}_\phi(x)=0$ implies $\nabla V(x)^\top \mathbf{v}_\phi(x)=0$.
Hence the set $V^{-1}(c)$ is forward-invariant for the flow of $\mathbf{v}_\phi$; in particular, $V(x(t))$ remains constant along trajectories confined to a level set.
\item[(iii)] On $\mathcal{M}=V^{-1}(0)$, the dynamics is \emph{non-stationary} in general:
since $\mathbf{s}_\theta(x,1)=0$ on $\mathcal{M}$, we have $\mathbf{m}(x)=\mathbf{v}_\phi(x)$.
Under (A2.ii) the mixed field has unit norm, $\|\mathbf{m}(x)\|=\|\mathbf{v}_\phi(x)\|=1$, so the motion on $\mathcal{M}$ persists with nonzero speed while preserving $V(x)\equiv 0$.
\end{itemize}
\end{theorem}

\begin{figure*}[htbp]
    \centering
    \begin{subfigure}[b]{0.24\linewidth}
        \centering
        \includegraphics[width=\linewidth]{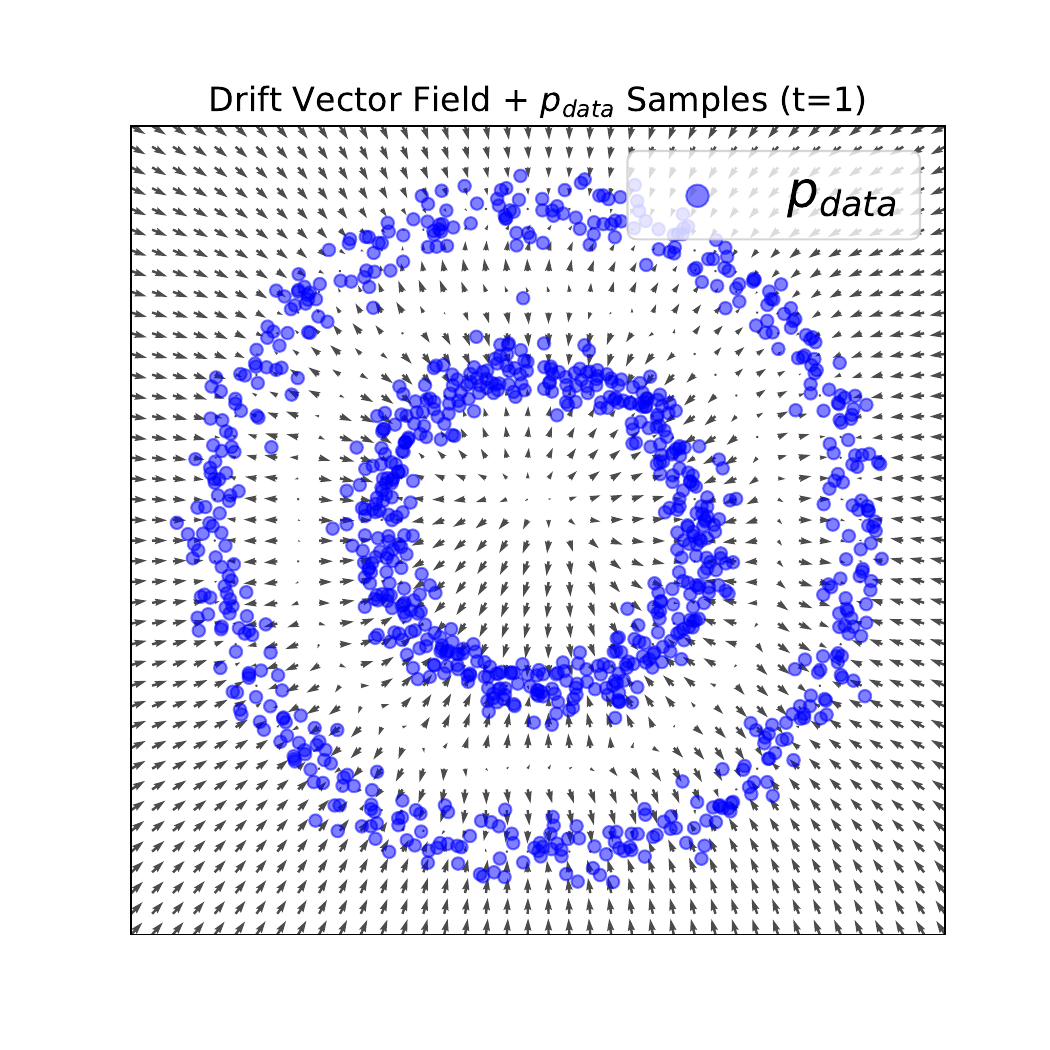}
        \caption{Learned score field for concentric circles path}
        \label{fig:doublescore}
    \end{subfigure}
    \hfill
    \begin{subfigure}[b]{0.24\linewidth}
        \centering
        \includegraphics[width=\linewidth]{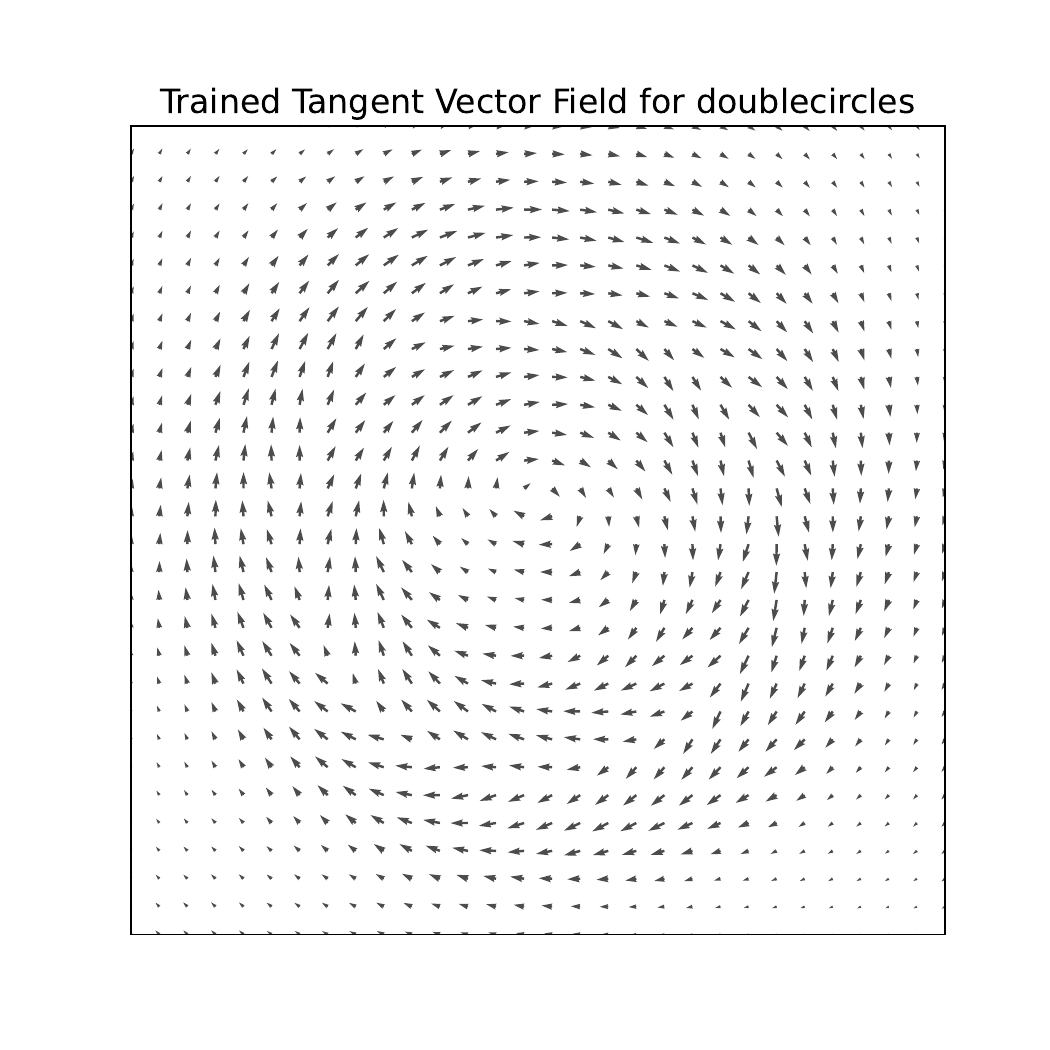}
        \caption{Learned tangent field for concentric circles path}
        \label{fig:tangent}
    \end{subfigure}
    \hfill
    \begin{subfigure}[b]{0.24\linewidth}
        \centering
        \includegraphics[width=\linewidth]{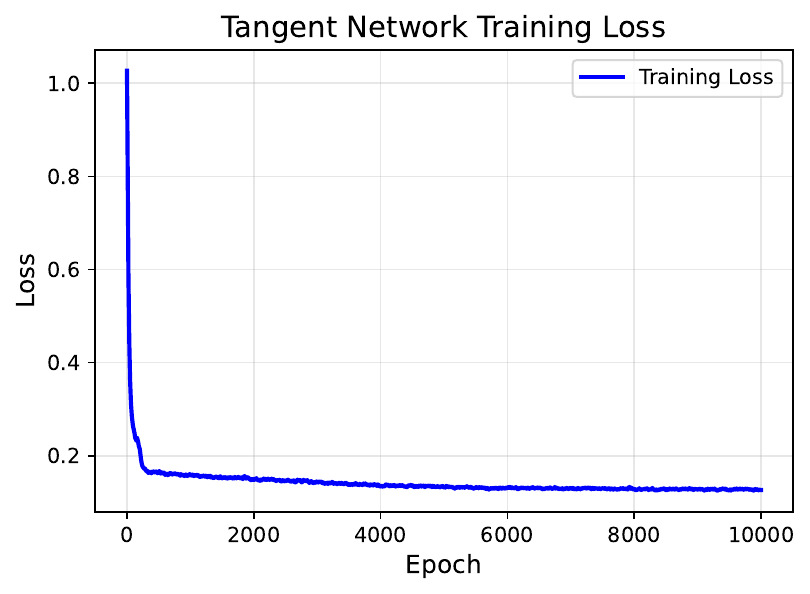}
        \caption{Tangent network training loss for concentric circles path}
        \label{fig:loss}
    \end{subfigure}
    \hfill
    \begin{subfigure}[b]{0.24\linewidth}
        \centering
        \includegraphics[width=\linewidth]{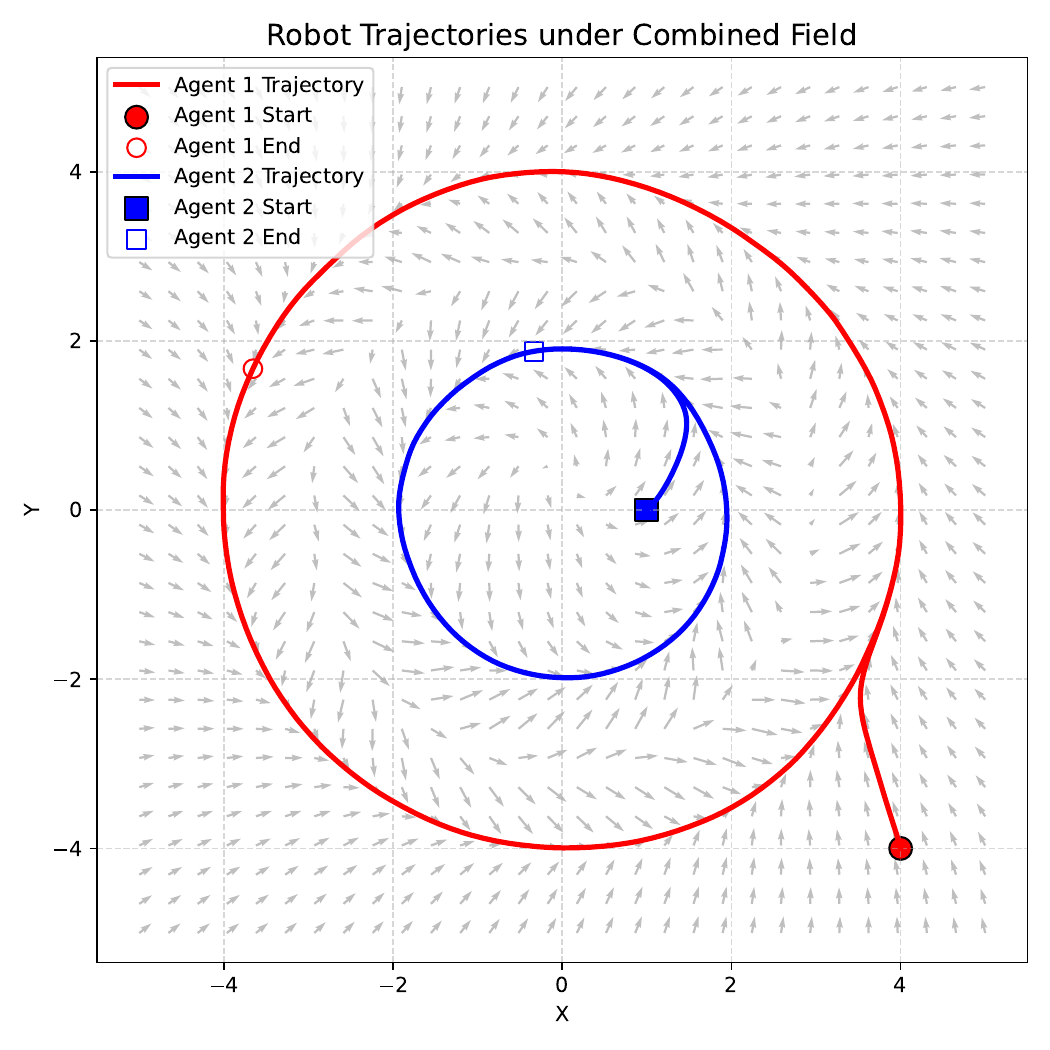}
        \caption{Path-following trajectories for concentric circles path}
        \label{fig:pathfollowing}
    \end{subfigure}

    \vspace{0.3cm} 
    
    \begin{subfigure}[b]{0.24\linewidth}
        \centering
        \includegraphics[width=\linewidth]{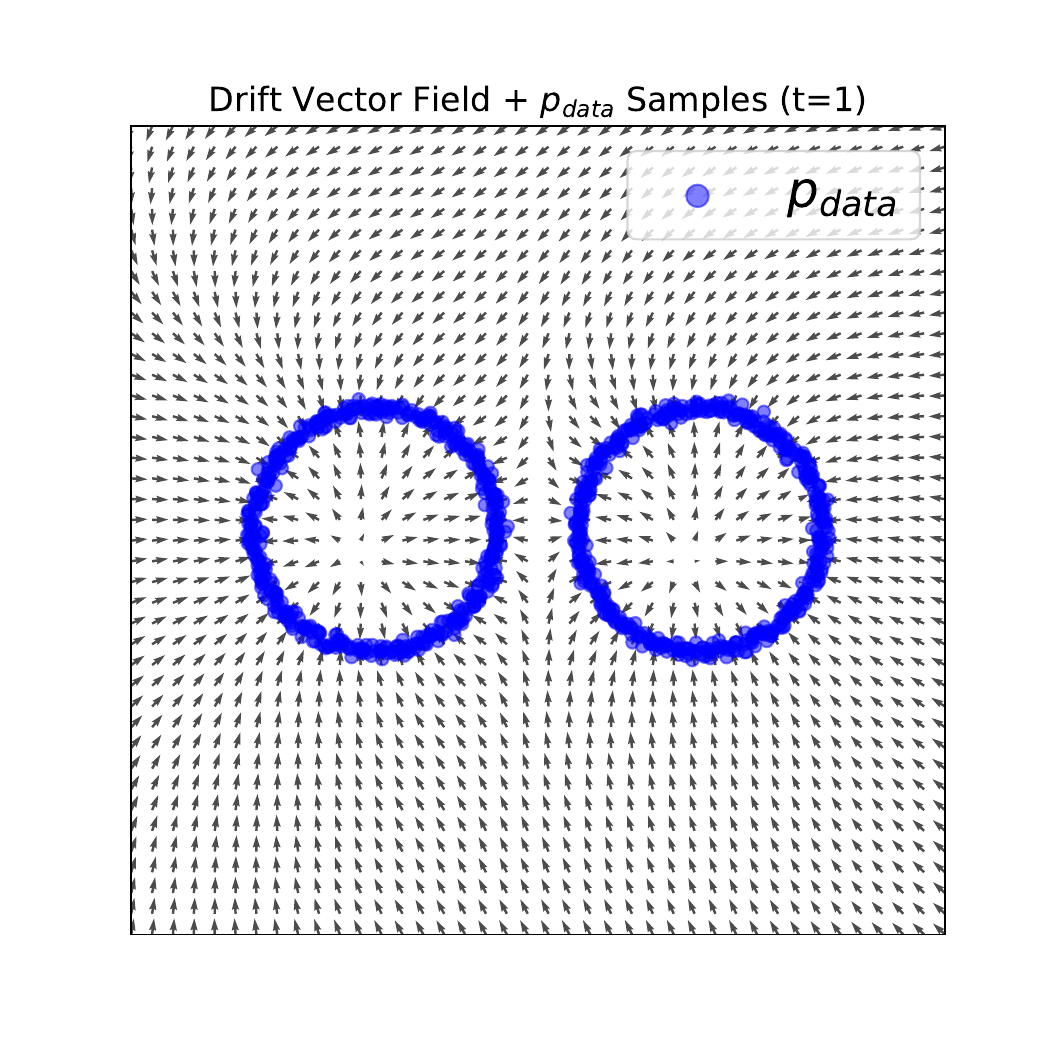}
        \caption{Learned score field for separately circles path}
        \label{fig:sscore}
    \end{subfigure}
    \hfill
    \begin{subfigure}[b]{0.24\linewidth}
        \centering
        \includegraphics[width=\linewidth]{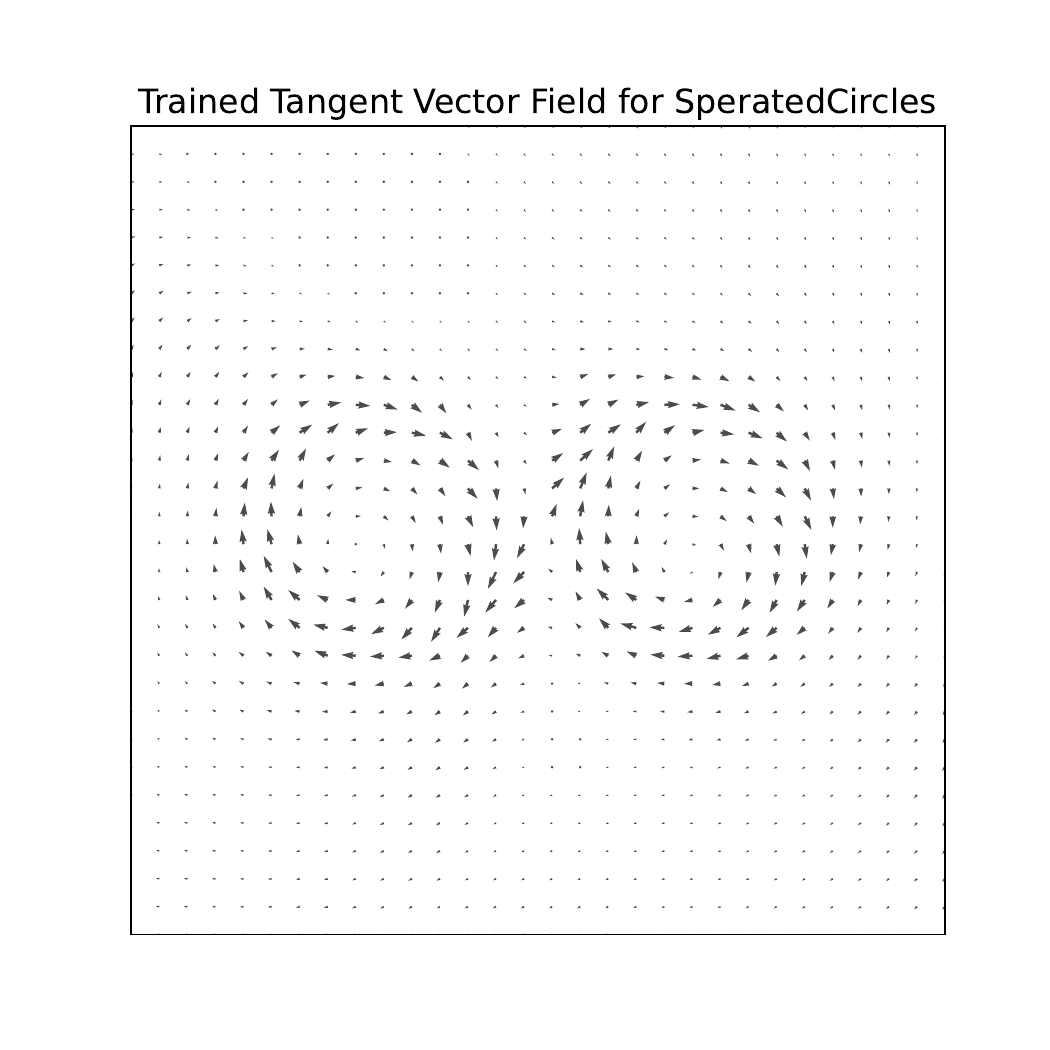}
        \caption{Learned tangent field for separately circles path}
        \label{fig:stangent}
    \end{subfigure}
    \hfill
    \begin{subfigure}[b]{0.24\linewidth}
        \centering
        \includegraphics[width=\linewidth]{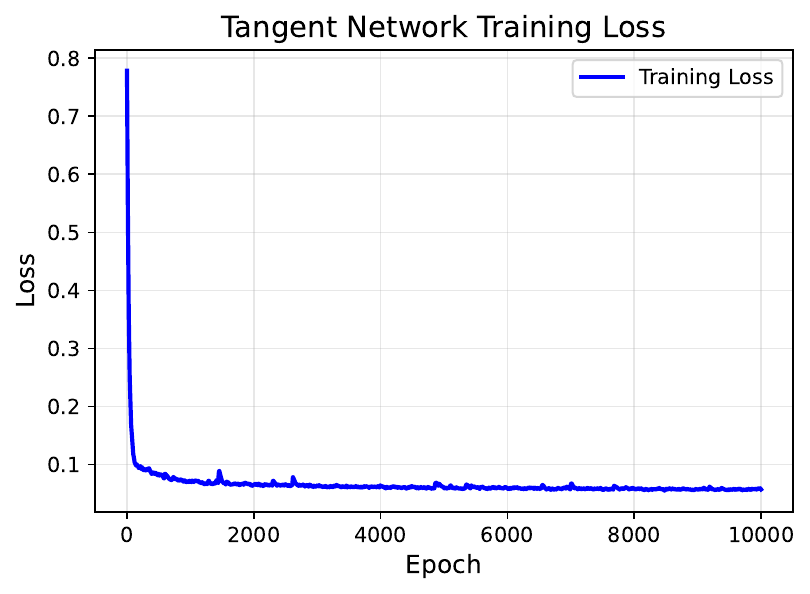}
        \caption{Tangent network training loss for separately circles path}
        \label{fig:sloss}
    \end{subfigure}
    \hfill
    \begin{subfigure}[b]{0.24\linewidth}
        \centering
        \includegraphics[width=\linewidth]{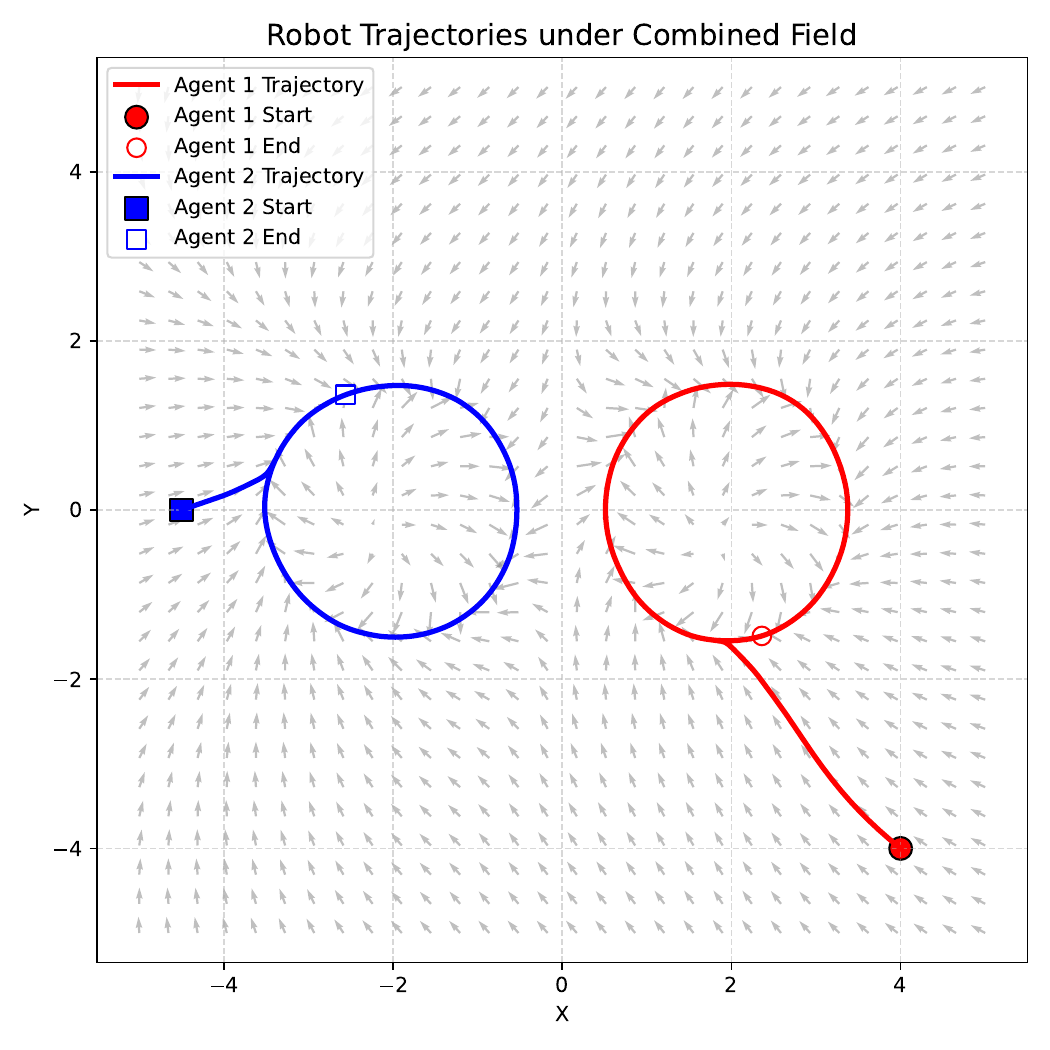}
        \caption{Path-following trajectories for separately circles path}
        \label{fig:spathfollowing}
    \end{subfigure}

    \caption{Visualization of the proposed SGVF framework on two waypoint-based path scenarios. 
    The top row corresponds to the \emph{concentric-circles} setup, showing the learned score field, tangent field, tangent network training loss, and the resulting path-following trajectories. 
    The bottom row corresponds to the \emph{separated-circles} setup, presenting the same components. }
    \label{fig:sgvf_overview}
\end{figure*}

\begin{proof}
(i) By Lemma~\ref{lem:decrease}, $\dot V\le 0$, so $V$ is Lyapunov.
LaSalle's invariance principle together with (A3) implies convergence to the largest invariant subset of $\mathcal{E}=\{\mathbf{s}_\theta=0\}$.

(ii) For any $x$ with fixed $V(x)=c$, the normal to the level set is $\nabla V(x)=-\mathbf{s}_\theta(x,1)$.
Orthogonality (A2.i) gives $\nabla V(x)^\top \mathbf{v}_\phi(x)=0$, i.e., $v$ is tangent to the level set, so $d(V\circ \gamma)/dt=0$ for trajectories $\gamma$ evolving under $\dot x=\mathbf{v}_\phi(x)$.
Since $\mathbf{m}=\mathbf{s}_\theta+\mathbf{v}_\phi$ and $\nabla V^\top \mathbf{s}_\theta=-\|\mathbf{s}_\theta\|^2$, the only contribution to $\dot V$ off the set $\{\mathbf{s}_\theta=0\}$ is strictly negative, while on $\{\mathbf{s}_\theta=0\}$ the contribution vanishes and $V$ is preserved.

(iii) On $\mathcal{M}$ we have $\mathbf{s}_\theta=0$ by first-order optimality at global modes of $P$.
Thus $\mathbf{m}=\mathbf{v}_\theta$ on $\mathcal{M}$.
Assumption (A2.ii) enforces $\|m\|=1$, hence the speed is nonzero and the dynamics is not stationary on $\mathcal{M}$, while by (ii) the motion remains on $V^{-1}(0)$.
\end{proof}

Robustness to soft constraints is important in practice, since orthogonality and unit-length are enforced via loss terms and therefore hold only approximately. Specifically, if the cosine error satisfies $|\mathbf{s}_\theta(x,1)^\top \mathbf{v}_\phi(x)| \le \varepsilon\,\|\mathbf{s}_\theta(x,1)\|\,\|\mathbf{v}_\phi(x)\|$ for some $\varepsilon\in[0,1)$, then
\[
\begin{aligned}
    \dot V(x) &= -\|\mathbf{s}_\theta(x,1)\|^2 - \mathbf{s}_\theta(x,1)^\top \mathbf{v}_\phi(x) \\
    &\le -\big(1-\varepsilon\|\mathbf{v}_\phi(x)\|\big)\,\|\mathbf{s}_\theta(x,1)\|^2.
\end{aligned}
\]
Consequently, for bounded $\|\mathbf{v}_\phi\|$ and sufficiently small $\varepsilon$, we have $\dot V<0$ whenever $\mathbf{s}_\theta \neq 0$, implying that the convergence to level sets and their invariance remain valid up to a neighborhood whose size is determined by the training accuracy.

\begin{figure*}[htbp]
    \centering
    \begin{subfigure}[b]{0.245\textwidth}
        \centering
        \includegraphics[width=\linewidth, angle=-90]{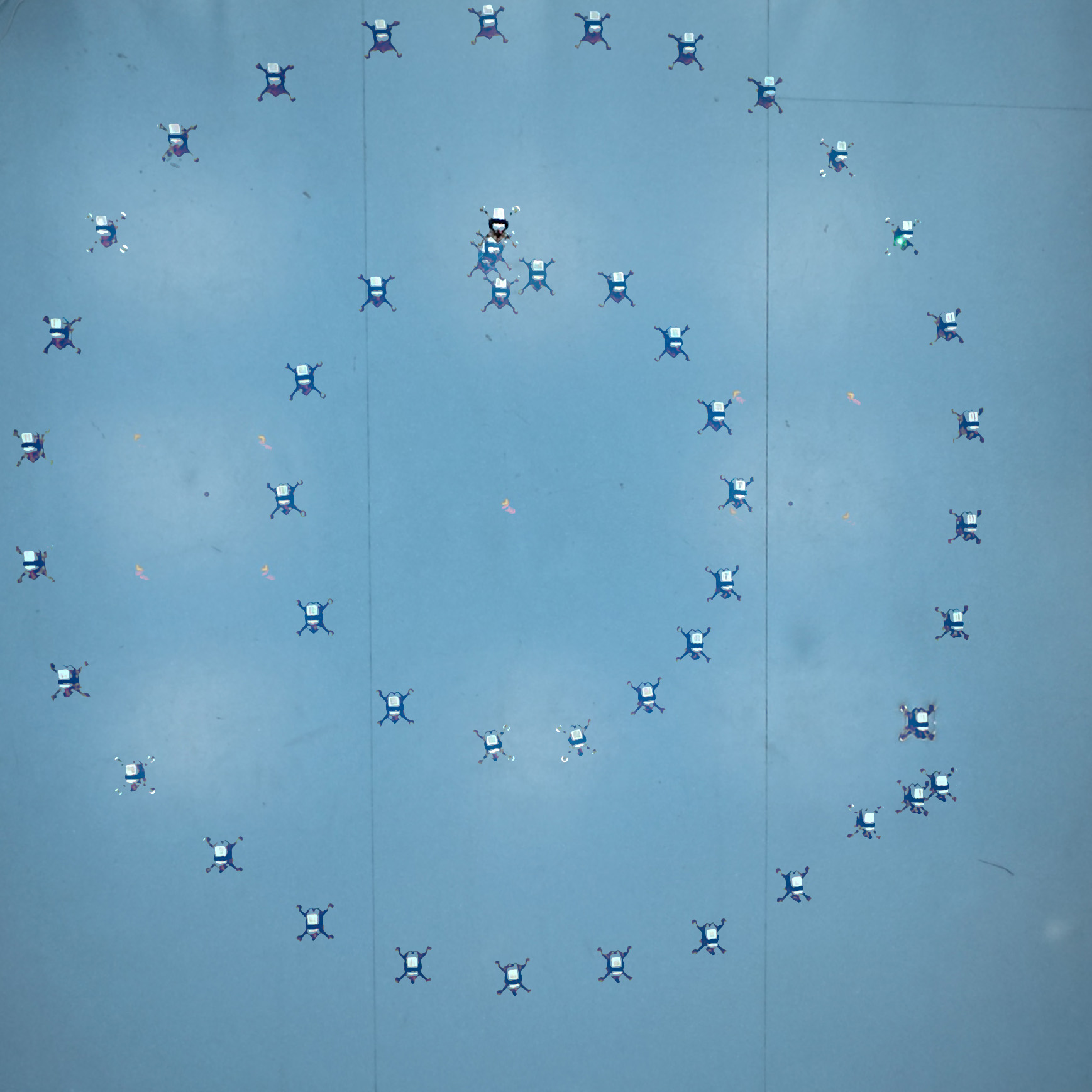}
        \caption{Overlayed images of the concentric circles experiment}
        \label{fig:ex1figure}
    \end{subfigure}
    \hfill
    \begin{subfigure}[b]{0.27\textwidth}
        \centering
        \includegraphics[width=\linewidth]{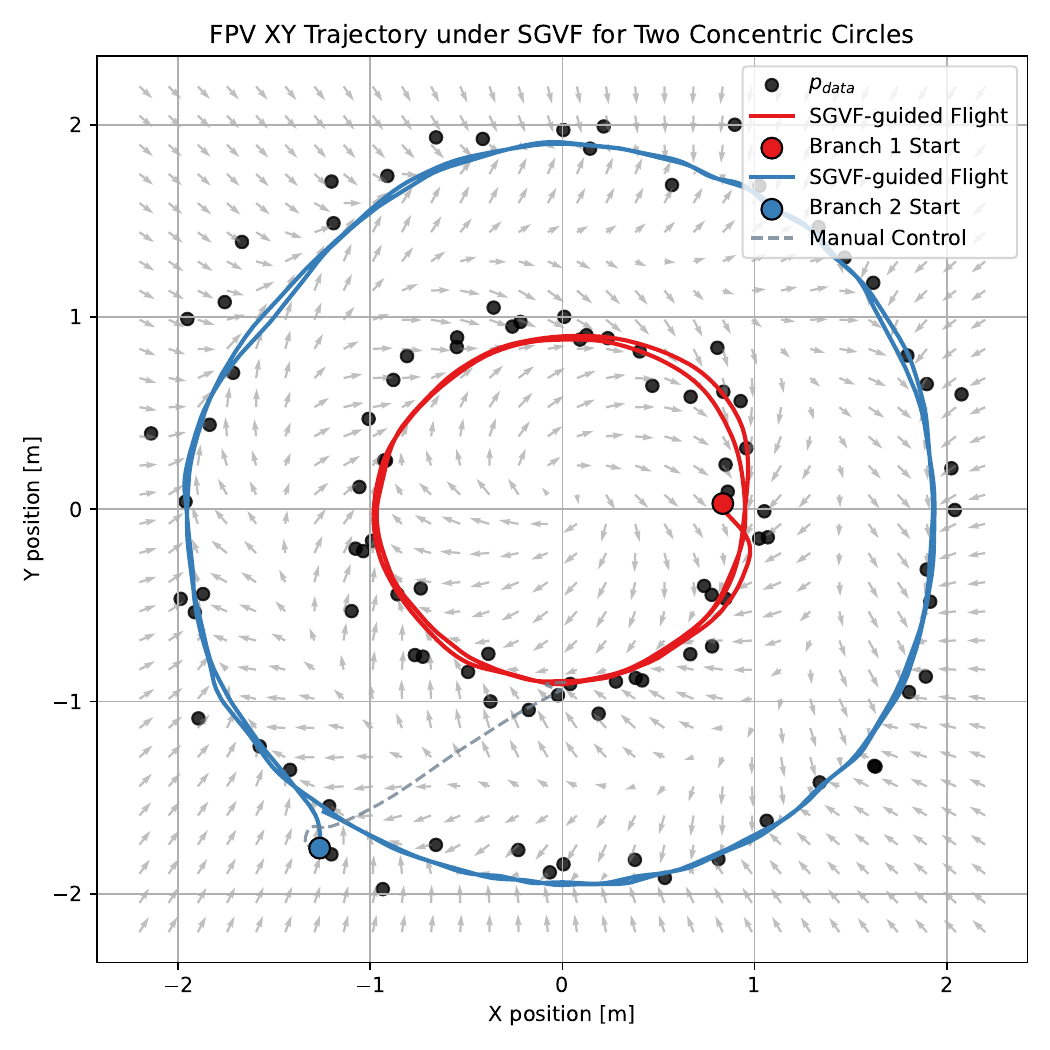}
        \caption{Trajectory plot with SGVF of concentric circles}
        \label{fig:ex1traj}
    \end{subfigure}
    \hfill
    \begin{subfigure}[b]{0.46\textwidth}
        \centering
        \includegraphics[width=\linewidth]{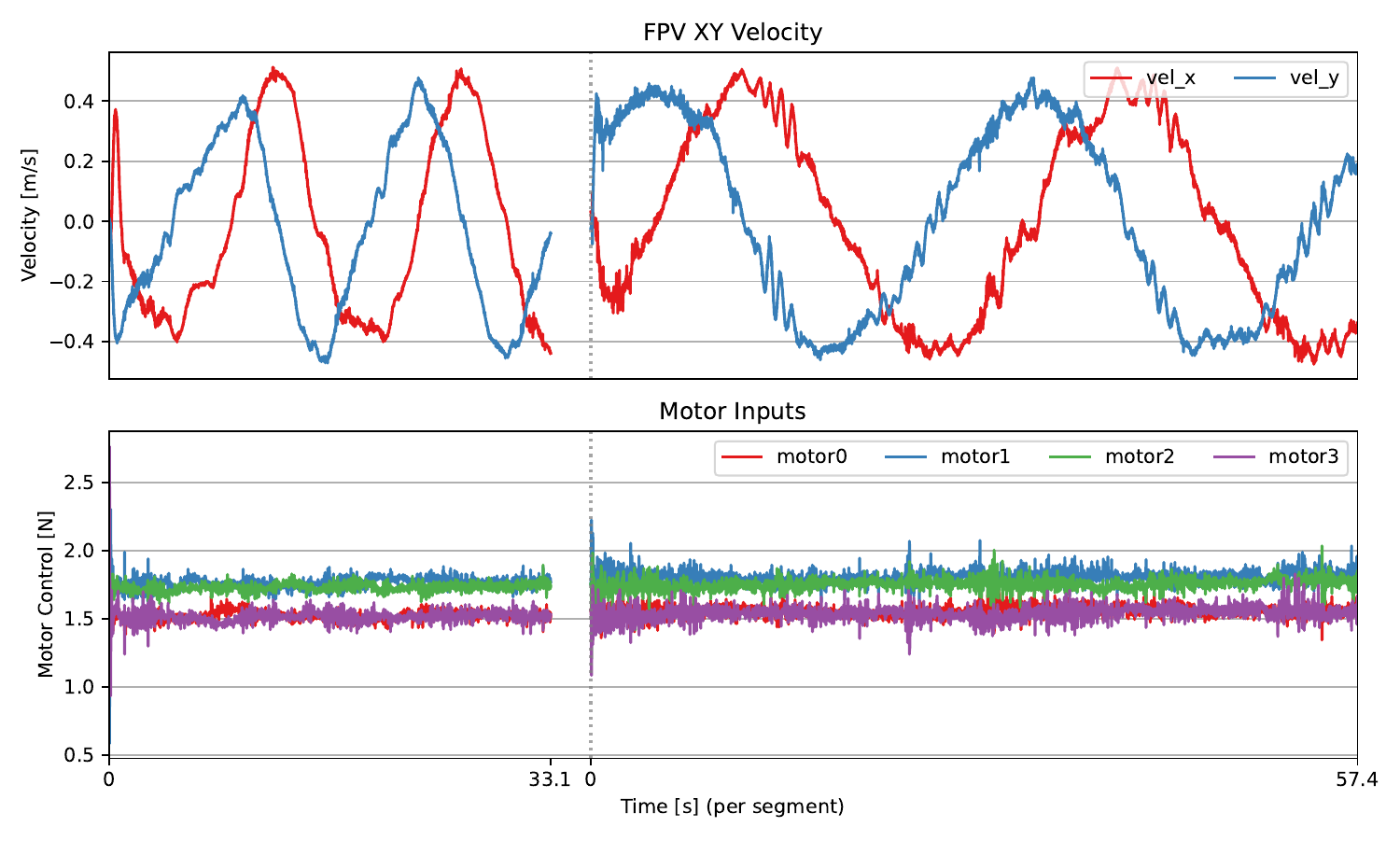}
        \caption{Velocity and control input plots in concentric circles experiment}
        \label{fig:ex1control}
    \end{subfigure}
    
    \vspace{0.3cm}

    \begin{subfigure}[b]{0.245\textwidth}
        \centering
        \includegraphics[width=\linewidth]{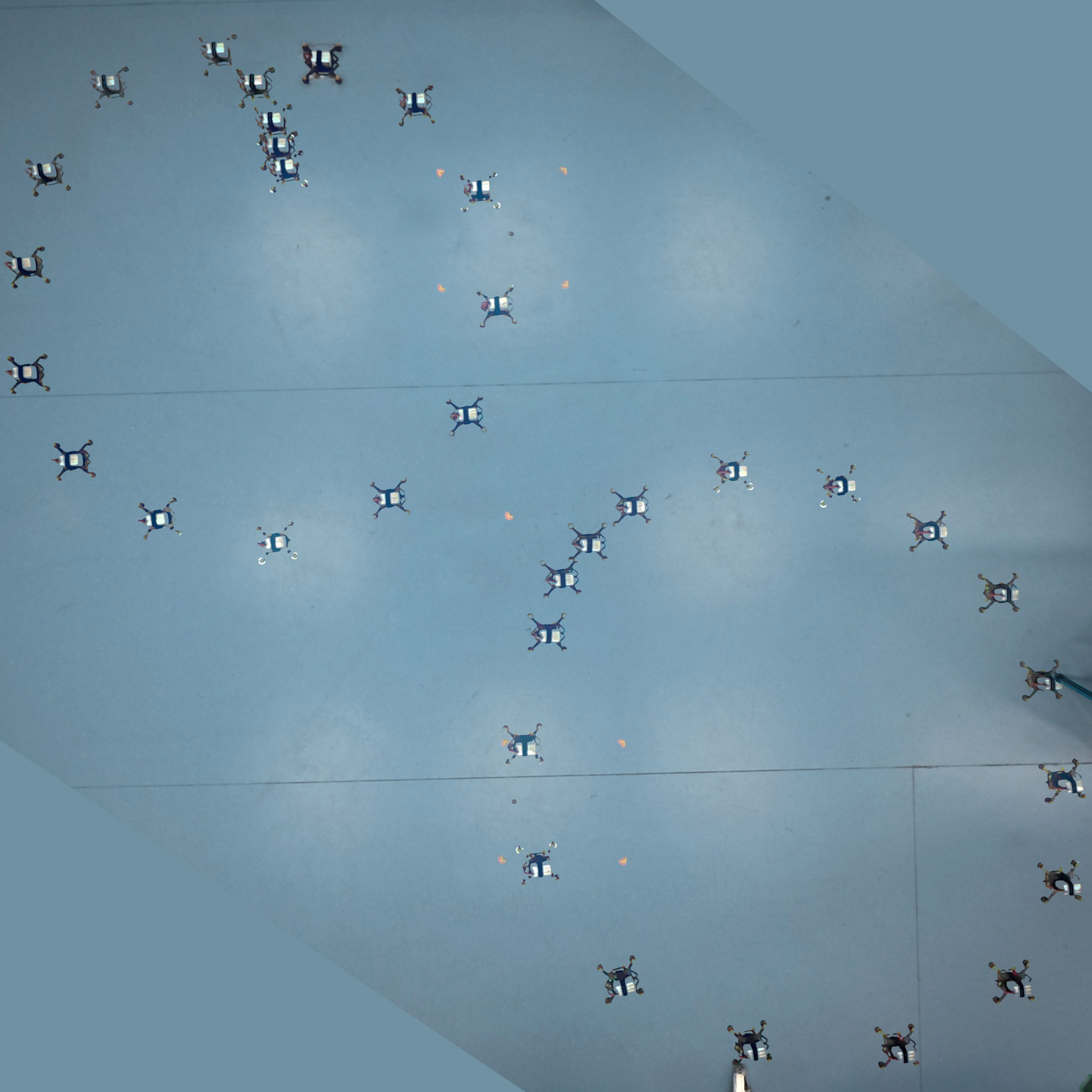}
        \caption{Overlayed images of the separated circles experiment}
        \label{fig:ex2figure}
    \end{subfigure}
    \hfill
    \begin{subfigure}[b]{0.27\textwidth}
        \centering
        \includegraphics[width=\linewidth]{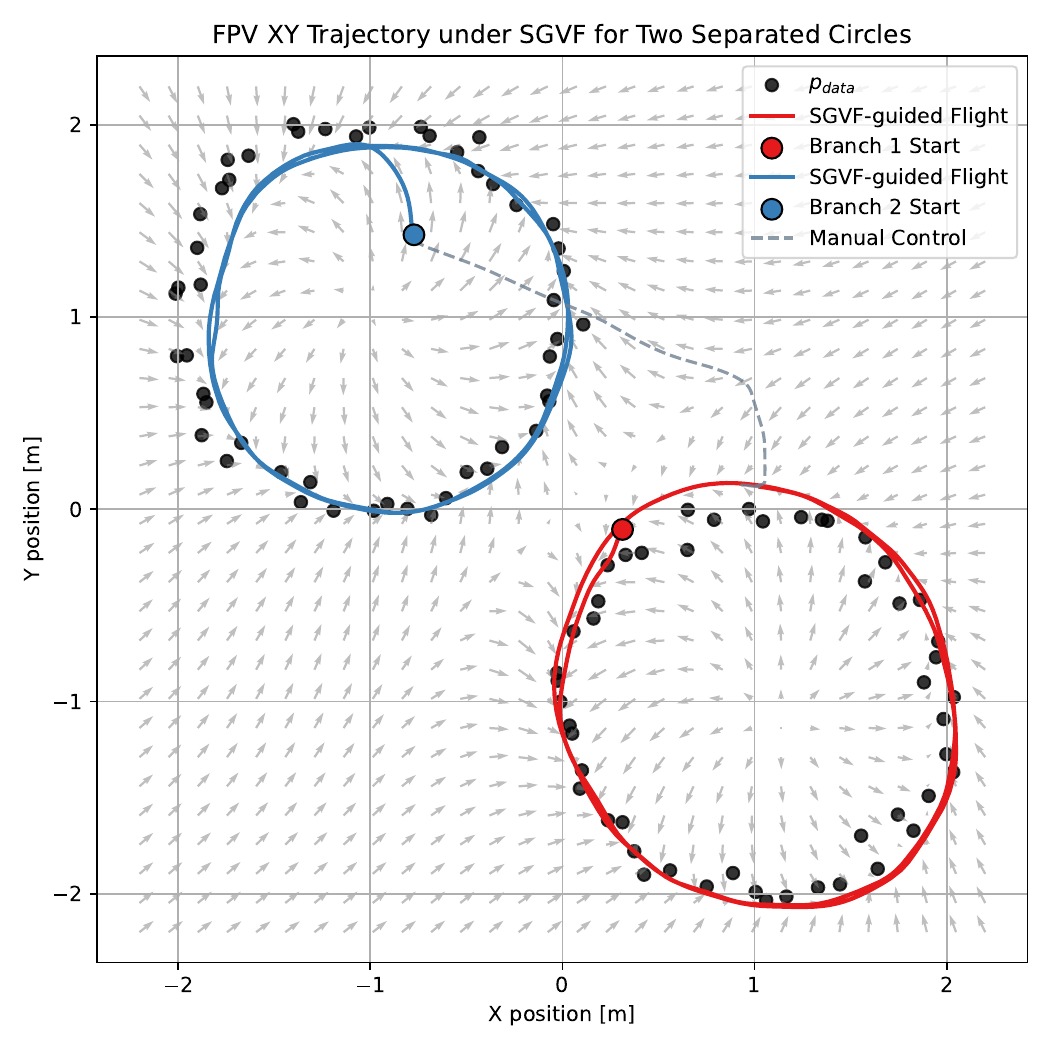}
        \caption{Trajectory plot with SGVF of separated circles}
        \label{fig:ex2traj}
    \end{subfigure}
    \hfill
    \begin{subfigure}[b]{0.46\textwidth}
        \centering
        \includegraphics[width=\linewidth]{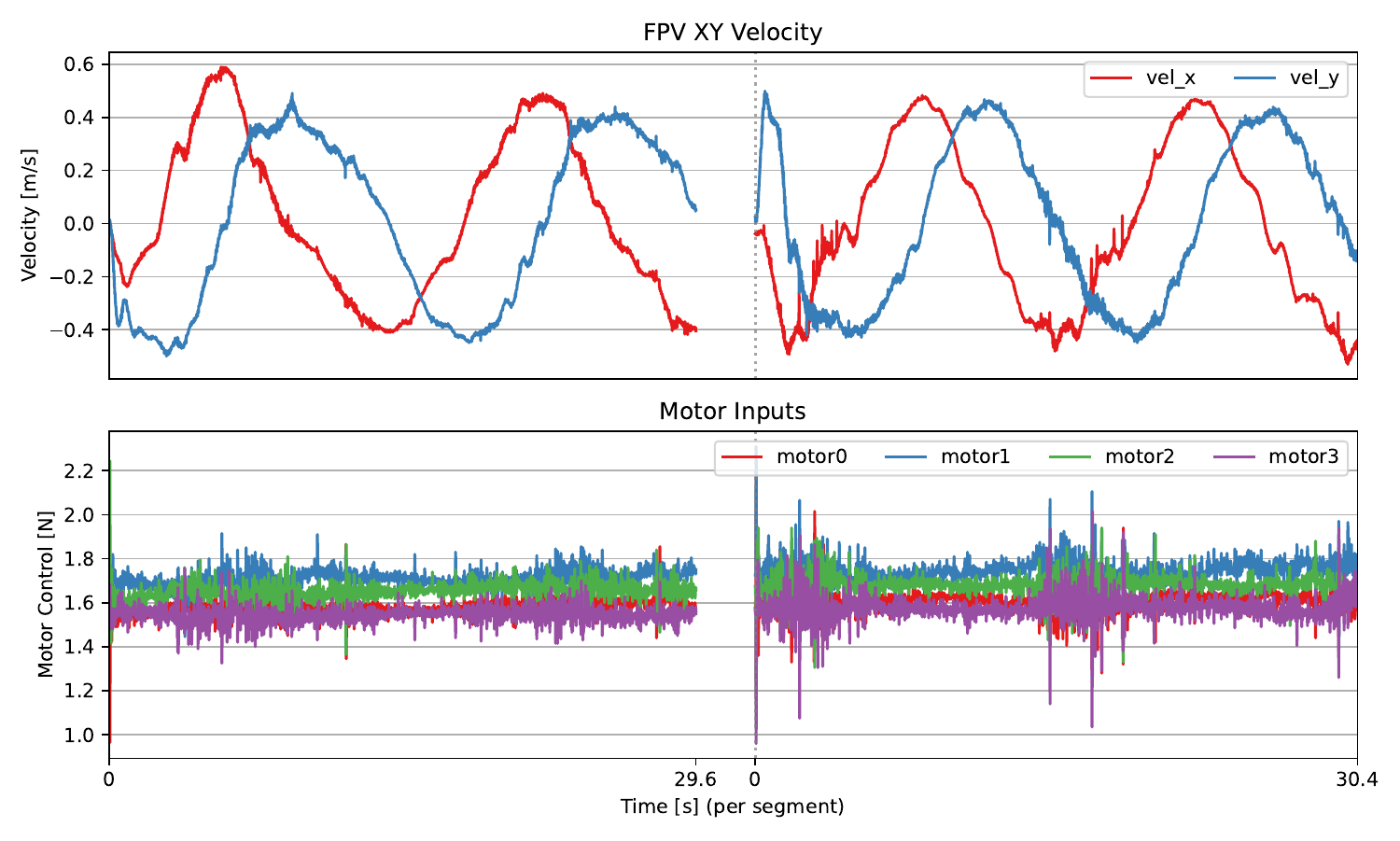}
        \caption{Velocity and control input plots in separated circles experiments}
        \label{fig:ex2control}
    \end{subfigure}
    
    \caption{Experimental results. The top row illustrates the \emph{concentric-circles} scenario, including the overlaid experiment video, the trajectory, and the velocity profile with motor inputs. The bottom row shows the \emph{separated-circles} scenario, presenting the same set of components.}
    \label{fig:experiment-results}
\end{figure*}

\section{Simulation and Experiment Results}
\subsection{Simulation Results}

To evaluate the effectiveness of the proposed SGVF, we conducted simulations on two representative waypoint-based planar path-following tasks: a \emph{double-circles} path and a \emph{separated-circles} path. As shown in Fig.~\ref{fig:sgvf_overview}, the learned score field captures the waypoint density distribution, while the tangent field provides directional guidance along the path. We employ $[3,64,64,64,64,2]$ and $[2,128,2]$ MLPs with SiLU activations for the score $\mathbf{S}_\theta(x,t)$ and tangent field $\mathbf{v}_\phi(x)$, respectively, with the parameter $k_s$ set to $0.2$. Both networks are trained using the Adam optimizer with a learning rate of $10^{-3}$. Training is performed for $10{,}000$ iterations with a batch size of $512$, and other hyperparameters are set as $k=5$ and $\lambda_1 = \lambda_2 = \lambda_3 = 1.0$. Model training and simulations were conducted using PyTorch on a Linux workstation equipped with an Intel i9-14900KF CPU and an NVIDIA RTX 4090 GPU. Two agents are used to follow the combined vector field from different starting points, testing the unified representation of the networks. 

The training losses of the tangent network demonstrate convergence and stability, indicating that the tangent component can be effectively learned on top of the score field. The resulting trajectories show that the agents successfully track complex waypoint-defined paths—including closed loops and disjoint branches—without requiring explicit parametric expressions of the path. These results validate the capability of SGVF to generalize across diverse geometrical structures and to integrate probabilistic attraction with directional guidance for robust waypoint-based navigation.

\subsection{Experiment Results}

For hardware validation, we evaluated the proposed method on two representative scenarios: concentric circles and separated circles, using a quadrotor UAV platform equipped with an RK3566 NPU for onboard computation. The MAV weighs 0.46 kg and has a motor-to-motor span of 0.149 m, with a thrust-to-weight ratio of 4.1. Low-level flight control was implemented via an $\mathrm{SE}(3)$-based controller. Onboard sensors sample voltage at 1000 Hz, while a Vicon motion capture system provides precise indoor localization by recording position, attitude, and velocity at 100 Hz in real time.

To evaluate the multi-branch performance, the UAV was manually switched to a new branch after completing the previous one. The experimental results are shown in Fig.~\ref{fig:experiment-results}. The outcome demonstrates that the SGVF successfully enables multi-branch path following. Under SGVF, the FPV flight velocity and motor inputs remain smooth and bounded, demonstrating the method’s good smoothness properties.

\subsection{Ablation Study}
\begin{figure}[htbp]
    \centering
    \begin{subfigure}[b]{0.32\linewidth}
        \centering
        \includegraphics[width=\linewidth]{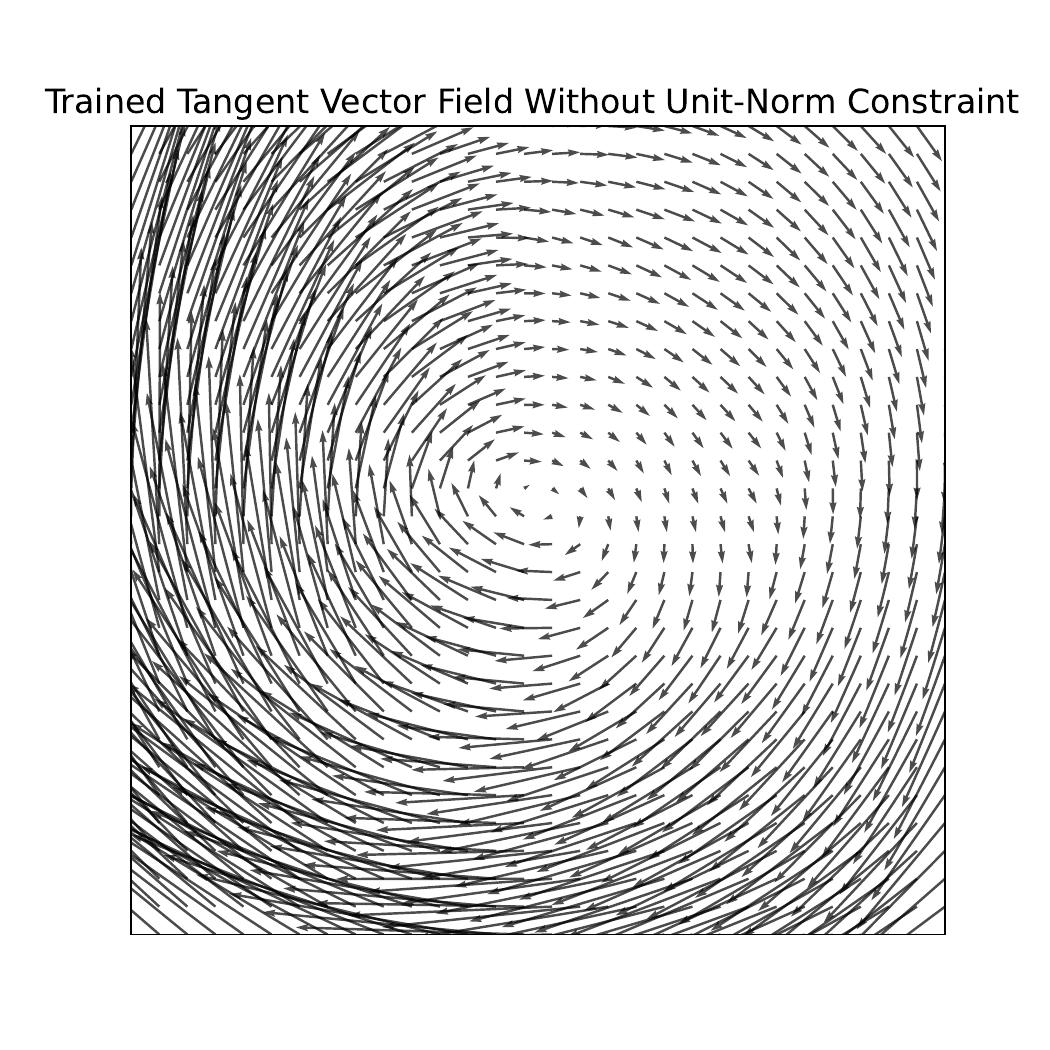}
        \caption{Remove unit-length constraint}
        \label{fig:Unit-Length Constraint}
    \end{subfigure}
    \hfill
    \begin{subfigure}[b]{0.32\linewidth}
        \centering
        \includegraphics[width=\linewidth]{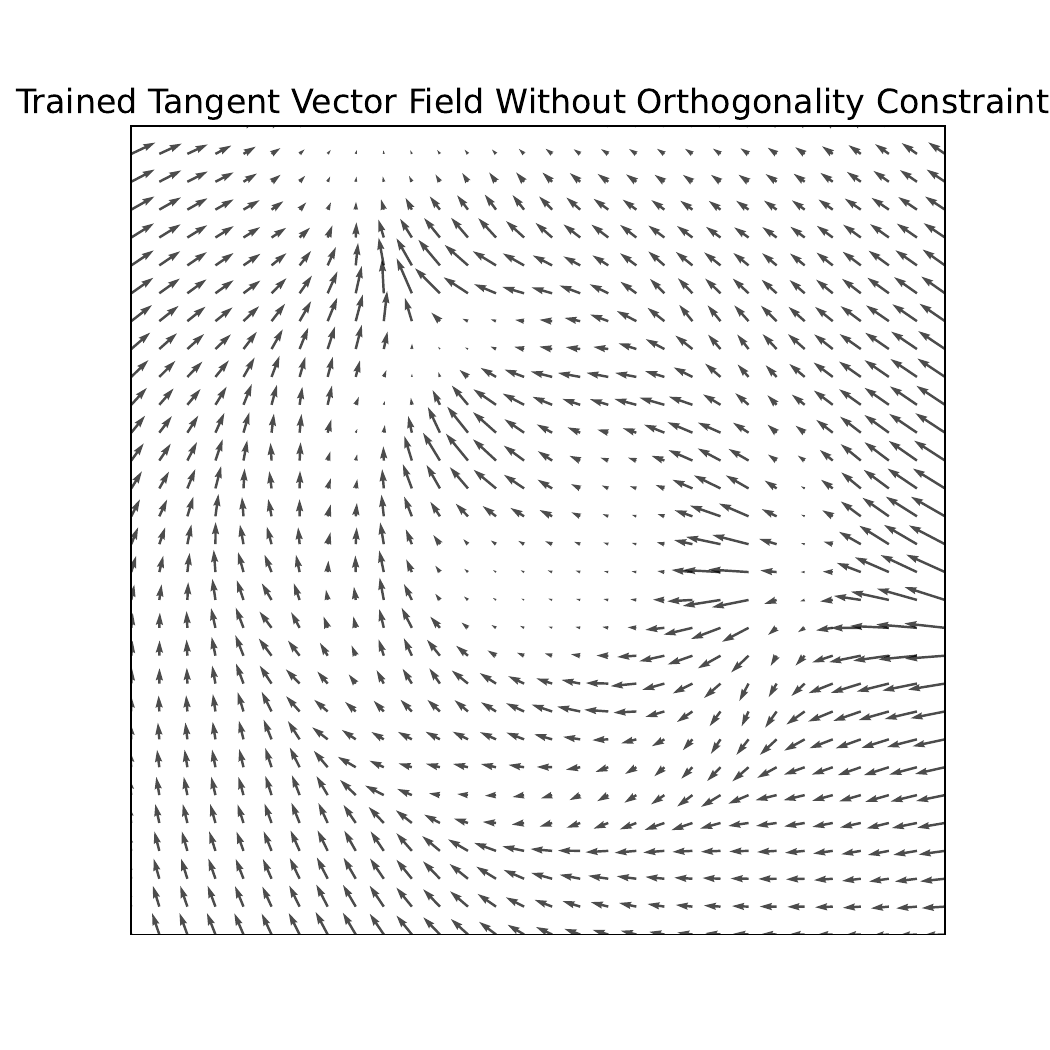}
        \caption{Remove orthogonal regularization}
        \label{fig:Orthogonal Regularization}
    \end{subfigure}
    \hfill
    \begin{subfigure}[b]{0.32\linewidth}
        \centering
        \includegraphics[width=\linewidth]{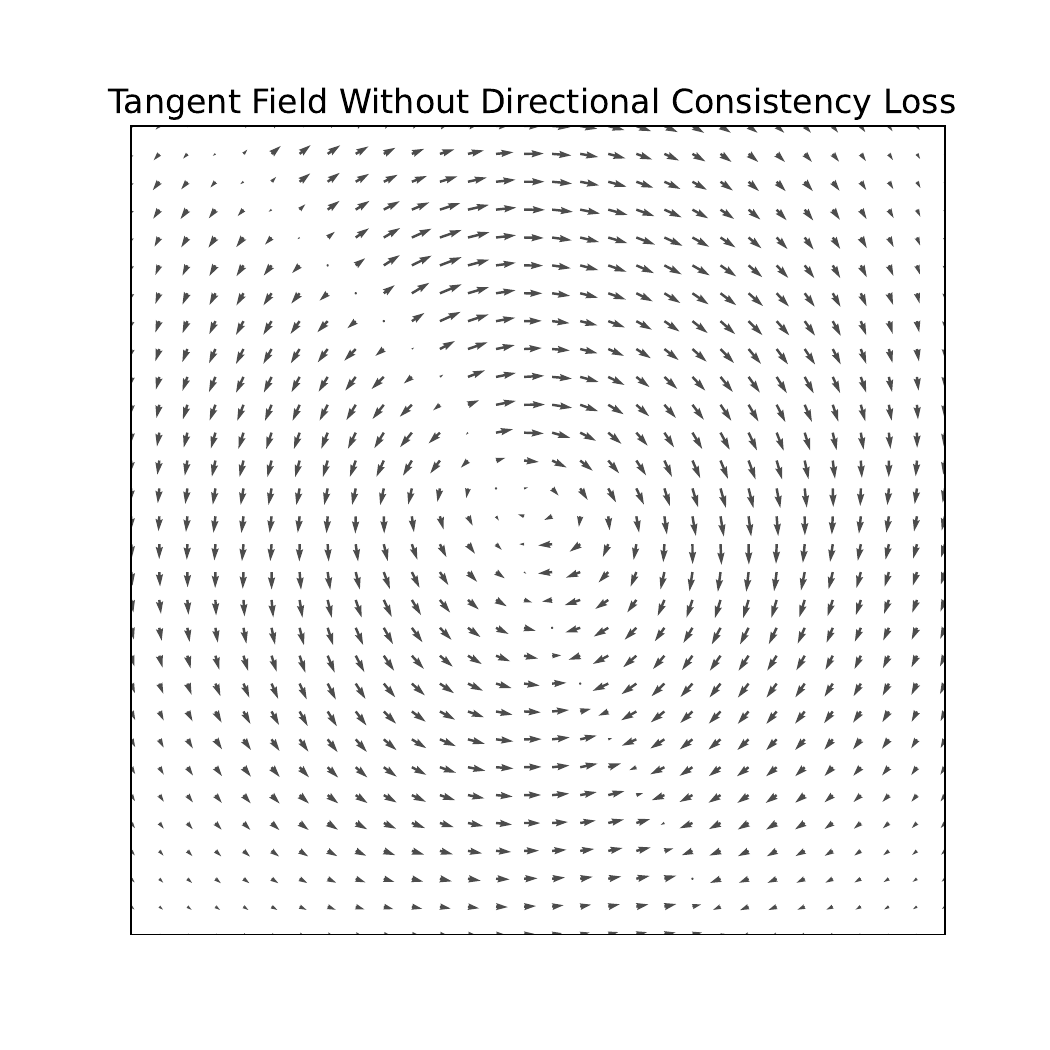}
        \caption{Remove directional constraint}
        \label{fig:Directional Consistency Constraint}
    \end{subfigure}
    \hfill
    \label{fig:ablationstudy}
    \caption{Ablation study on loss components}
\end{figure}

This section presents an ablation study on the different components of the loss function, using the concentric double-circle scenario (the first simulation case) to analyze the effect of removing each loss term. First, when the unit-length constraint is removed (see Fig. \ref{fig:Unit-Length Constraint}), the tangent vector field exhibits uncontrolled magnitudes, and the expected behavior—where the tangent dominates near the data points and diminishes away from them—is disrupted. Second, removing the orthogonality constraint leads to a tangent field that no longer follows the data points and instead tends to diverge away, indicating that the orthogonality term is crucial for properly combining the tangent with the score field (see Fig. \ref{fig:Orthogonal Regularization}). Finally, without the directional consistency constraint, the tangent field shows two conflicting directions, with one line clearly behaving as a source and another as a sink, demonstrating the role of directional consistency in maintaining field coherence and stability (see Fig. \ref{fig:Directional Consistency Constraint}). These ablation experiments provide an intuitive understanding of how each loss term contributes to training a reliable guiding vector field.

\section{Discussion}
\subsection{Score Vanishing and Singularities in Learned Vector Fields}
In score-based generative modeling, the learned score function $\nabla_x \log p_t(x)$ implicitly defines a guidance vector field that transports noise samples toward the data manifold. This field, though learned from data, behaves analogously to control-theoretic GVFs that are designed to steer agents toward predefined paths. Notably, in both settings, vector field singularities often arise—regions where the vector magnitude vanishes or becomes ill-defined (see Fig. \ref{fig:Learned Score Field}).
\begin{figure}[htbp]
    \centering
    \begin{subfigure}[b]{0.48\linewidth}
        \centering
        \includegraphics[width=\linewidth]{figure/doublecircles/score.png}
        \caption{Learned score field}
        \label{fig:Learned Score Field}
    \end{subfigure}
    \hfill
    \begin{subfigure}[b]{0.48\linewidth}
        \centering
        \includegraphics[width=\linewidth]{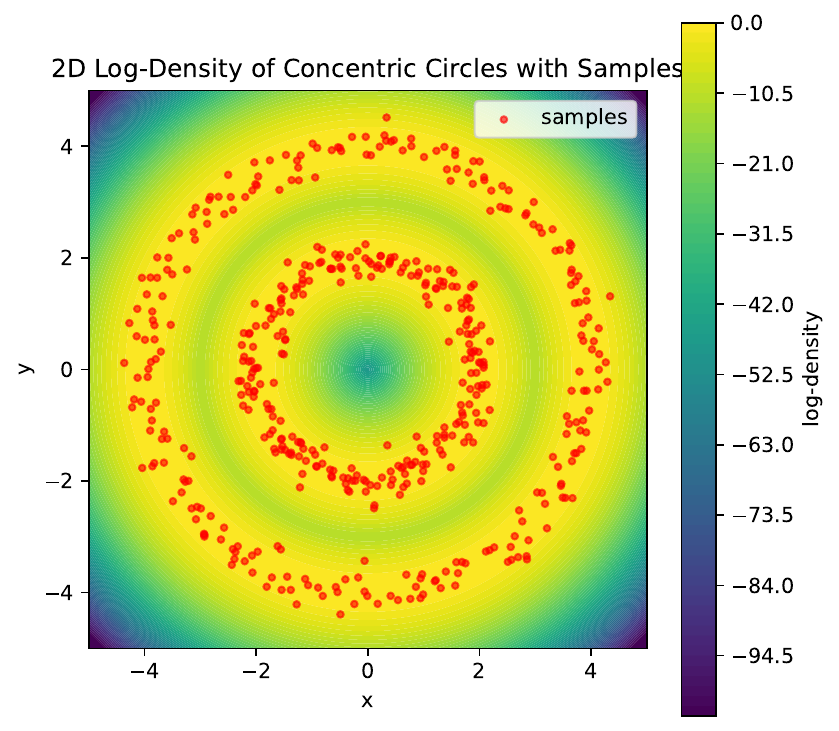}
        \caption{Log density in 2D}
        \label{fig:Log Density in 2D}
    \end{subfigure}
    \vspace{0.3cm}
    \begin{subfigure}[b]{0.48\linewidth}
        \centering
        \includegraphics[width=\linewidth]{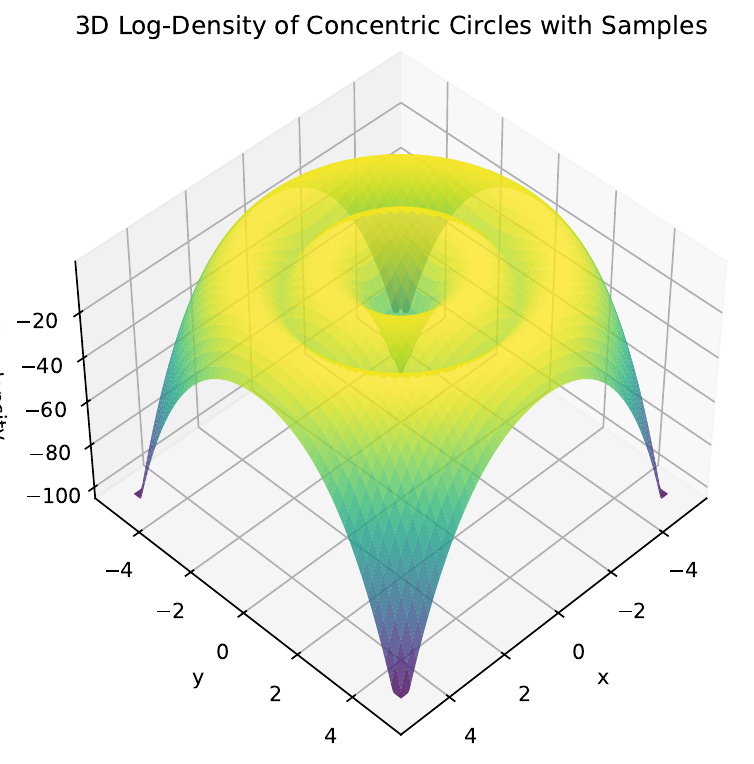}
        \caption{Log density in 3D}
        \label{fig:Log density in 3D}
    \end{subfigure}
    \hfill
    \begin{subfigure}[b]{0.48\linewidth}
        \centering
        \includegraphics[width=\linewidth]{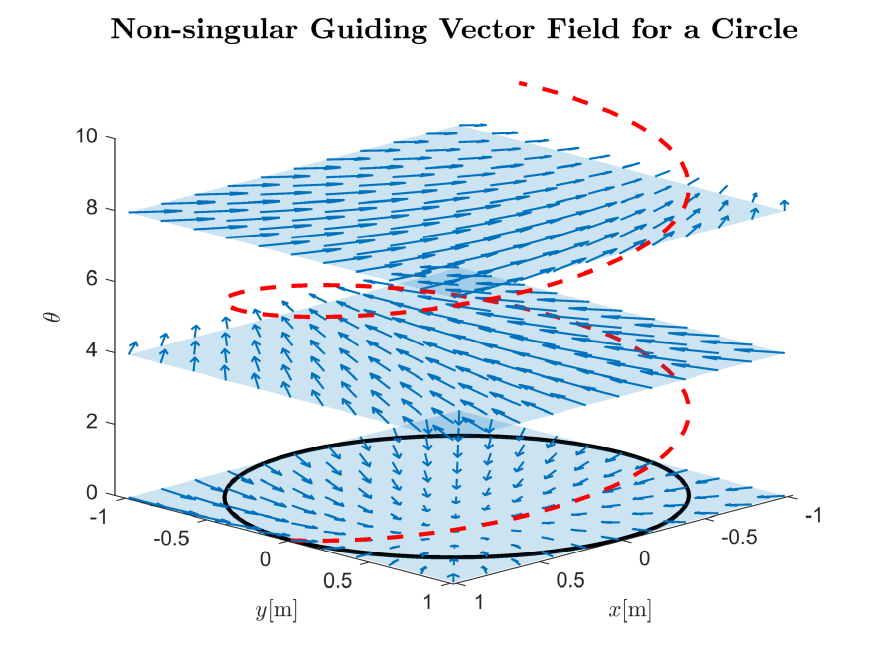}
        \caption{Non-singular guiding vector field \cite{chen2025non}}
        \label{fig:Non-singular guiding vector field}
    \end{subfigure}
    \label{fig:vanishing}
    \caption{Score vanishing and singularities in guiding vector field}
\end{figure}

When the target distribution is supported on a lower-dimensional manifold—particularly when represented as a discrete point cloud rather than an analytic form—the learned score field is forced to resolve topological inconsistencies in the data geometry. These inconsistencies are akin to the unavoidable singularities in GVF-based path-following control, which are constrained by the Poincaré-Hopf theorem. The theorem implies that for vector fields on manifolds, the total index of isolated singularities is a topological invariant (e.g., the Euler characteristic of the underlying manifold).

We propose that score vanishing—the phenomenon where the score norm approaches zero in localized regions—is not merely a training pathology, but a degenerate form of a topological singularity. It reflects the model's attempt to reconcile vector continuity and divergence constraints with the geometric structure of the data. These degeneracies are particularly pronounced near areas of low sampling density, high curvature, or nontrivial topology (e.g., loops, branches, or holes in the data support, see Fig. \ref{fig:Log Density in 2D},\ref{fig:Log density in 3D}).

This perspective offers a new geometric interpretation of score model failure modes and motivates the development of topology-aware regularization mechanisms. By drawing from the theory of vector field design in control systems—where singularity avoidance and minimal-index fields are active research topics (see Fig.\ref{fig:Non-singular guiding vector field})—we may be able to suppress undesirable degeneracies in learned score fields and improve training stability and expressivity.

\subsection{Conner Performance}
\begin{figure*}[htbp]
    \centering
    \begin{subfigure}[b]{0.32\linewidth}
        \centering
        \includegraphics[width=\linewidth]{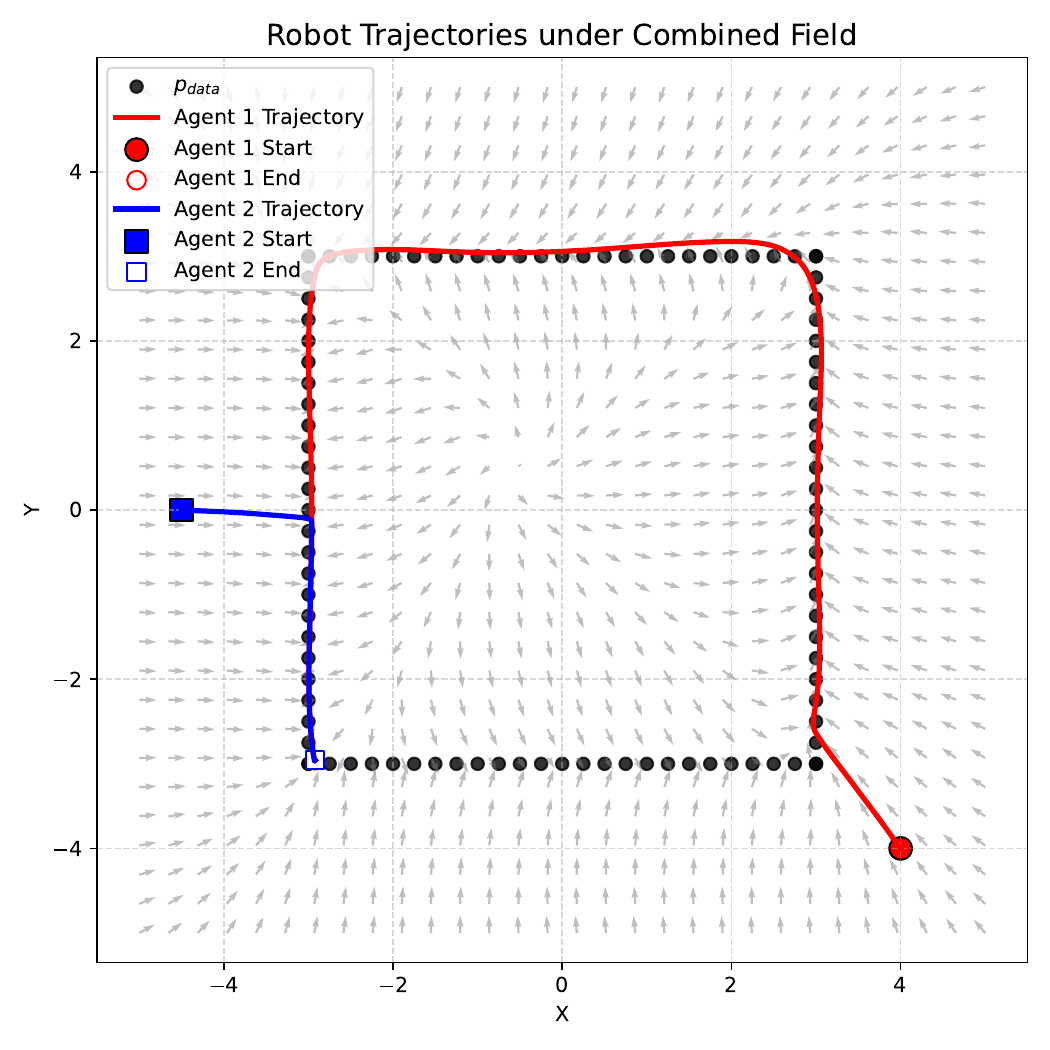}
        \caption{Square path, shallow network}
        \label{fig:square path}
    \end{subfigure}
    \hfill
    \begin{subfigure}[b]{0.32\linewidth}
        \centering
        \includegraphics[width=\linewidth]{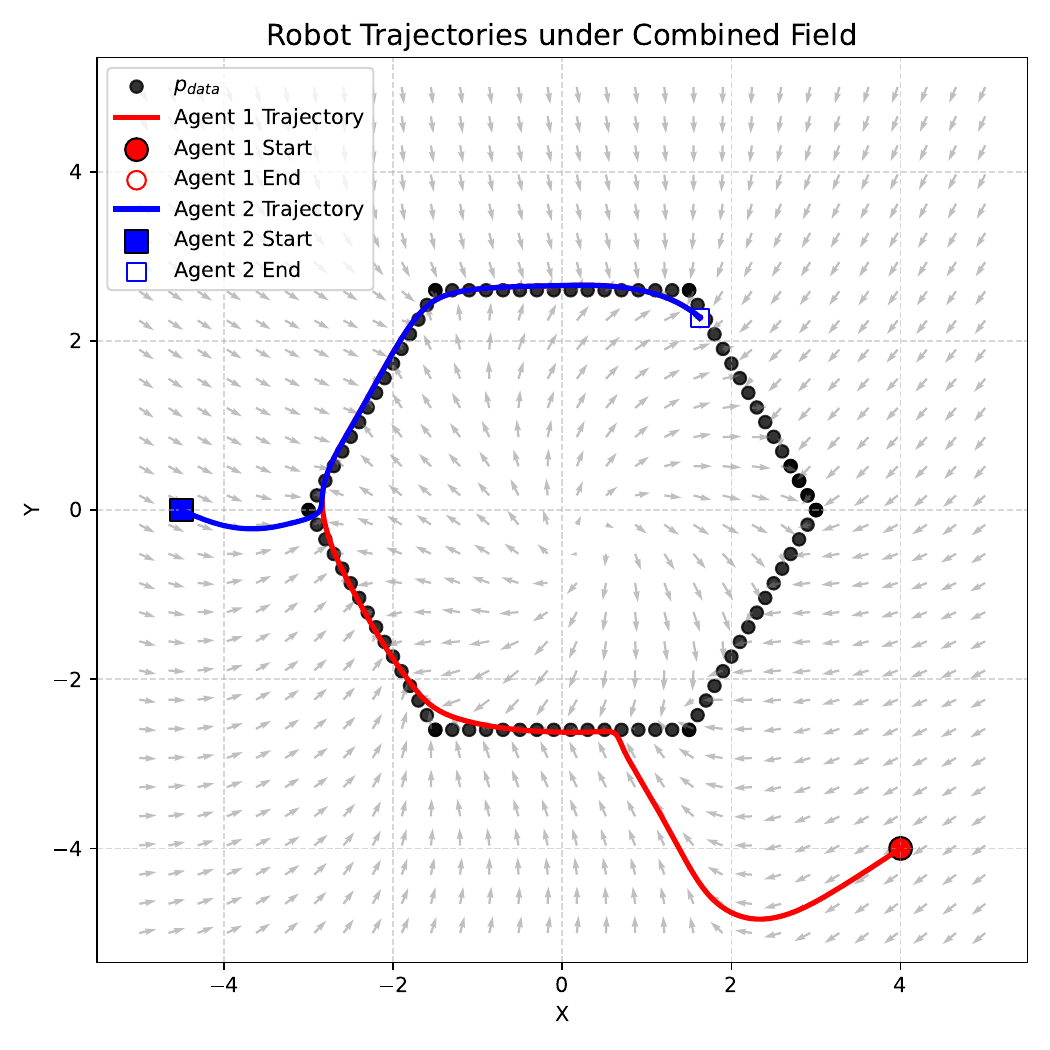}
        \caption{Hexagon path, shallow network}
        \label{fig:hexagon path, S}
    \end{subfigure}
    \hfill
    \begin{subfigure}[b]{0.32\linewidth}
        \centering
        \includegraphics[width=\linewidth]{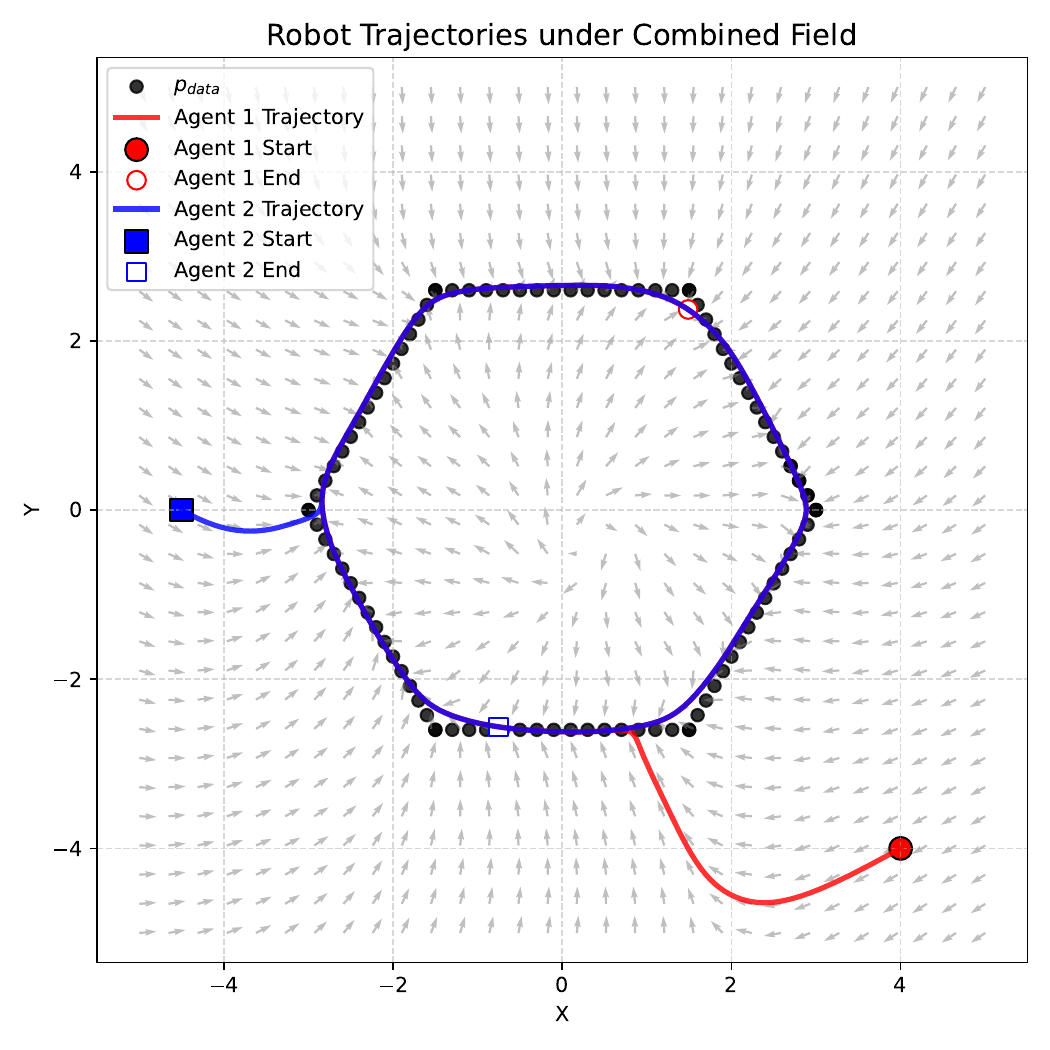}
        \caption{Hexagon path, deeper network}
        \label{fig:hexagon, D}
    \end{subfigure}
    \hfill
    \label{fig:corner}
    \caption{Path-following performance at polygon corners}
\end{figure*}
In this study, the guiding vector field is constructed using a score-guided soft regularization approach, which avoids imposing hard physical constraints on the learned field. This design offers several advantages, including faster convergence during training, a smoother optimization landscape, and effective guidance of an agent along waypoint-defined paths. However, the soft-guided formulation has limitations in handling regions with sharp geometric features, such as corners. Specifically, for polygonal paths with acute or moderately large angles, a shallow tangent network may generate local attractors near the corners, temporarily trapping the agent and causing deviations or pauses in its trajectory. Increasing the network depth effectively mitigates this issue, eliminating corner attraction and improving path-following performance, underscoring the influence of network capacity in representing complex geometric structures.

We investigate this behavior using two representative polygonal paths, including square and hexagonal shapes. For the square path (see Fig.\ref{fig:square path}), each corner forms a 90° angle, and simulations with a $[2,128,2]$ tangent network show that the agent occasionally stalls near these vertices due to local attractors in the tangent field. The hexagonal path, with 120° corners, exhibits similar behavior (see Fig.\ref{fig:hexagon path, S}) when the tangent network is shallow, though the trapping effect is less pronounced than in the square path. In contrast, a deeper $[2,64,64,64,64,2]$ network eliminates corner attraction entirely (see Fig.~\ref{fig:hexagon, D}). These observations highlight a trade-off inherent in the score-guided soft regularization approach: while it efficiently captures general path-following behavior and provides smooth guidance along gentle curves, its capacity to represent sharp corners or tightly curved segments is limited, suggesting the potential need for more expressive networks when handling high-curvature paths.

\section{Conclusion and Future Work}

In this work, we introduced SGVF, a unified framework that leverages score-based generative modeling to construct vector fields over unordered, multi-branch, and probabilistic paths. We demonstrated that SGVF can recover the tangent bundle of a manifold from discrete point clouds and provide robust guidance in scenarios where classical GVFs fail. While the current implementation relies on a basic score-based diffusion model with limited expressive capacity, recent advances such as the mean-flow model \cite{Geng_2025_MeanFlow_arxiv} suggest promising directions for improving efficiency, training stability, and fine-structure representation. More broadly, SGVF bridges generative modeling and geometric control, offering a foundation for future exploration of distribution-aware guidance in robotics. As part of ongoing research, we plan to extend these concepts to collective behavior identification and control, integrating learned interaction structures to guide multi-agent systems within a unified framework.

\bibliographystyle{IEEEtran}
\bibliography{sn-bibliography}

\end{document}